\documentclass{article}

\usepackage[final]{nips_2016}

\usepackage[utf8]{inputenc} 
\usepackage[T1]{fontenc}    
\usepackage{hyperref}       
\usepackage{url}            
\usepackage{booktabs}       
\usepackage{amsfonts}       
\usepackage{nicefrac}       
\usepackage{microtype}      
\usepackage{amsmath,amsthm}
\usepackage{mathtools}
\usepackage{graphicx}
\usepackage{epstopdf}
\usepackage{caption}
\usepackage[labelformat=parens,labelfont={}]{subcaption}
\usepackage{comment}

\newcommand{\defeq}{\vcentcolon=}
\newcommand{\eqdef}{=\vcentcolon}
\newcommand{\norm}[1]{\left\lVert#1\right\rVert}

\newtheorem{theorem}{Theorem}[section]
\newtheorem{lemma}[theorem]{Lemma}
\newtheorem{proposition}[theorem]{Proposition}

\theoremstyle{definition}

\newtheorem{defn}[theorem]{Definition}

\theoremstyle{remark}

\title{Dynamic matrix recovery from incomplete observations under an exact low-rank constraint}

\author{
  Liangbei Xu \: \: \: Mark A.~Davenport \\
  Department of Electrical and Computer Engineering\\
  Georgia Institute of Technology\\
  Atlanta, GA 30318 \\
  {lxu66@gatech.edu \: \:\: mdav@gatech.edu} \\
}
\begin{document}

\maketitle

\begin{abstract}
 Low-rank matrix factorizations arise in a wide variety of applications -- including recommendation systems, topic models, and source separation, to name just a few.  In these and many other applications, it has been widely noted that by incorporating temporal information and allowing for the possibility of time-varying models, significant improvements are possible in practice. However, despite the reported superior empirical performance of these dynamic models over their static counterparts, there is limited theoretical justification for introducing these more complex models. In this paper we aim to address this gap by studying the problem of recovering a dynamically evolving low-rank matrix from incomplete observations. First, we propose the locally weighted matrix smoothing (LOWEMS) framework as one possible approach to dynamic matrix recovery. We then establish error bounds for LOWEMS in both the {\em matrix sensing} and {\em matrix completion} observation models. Our results quantify the potential benefits of exploiting dynamic constraints both in terms of recovery accuracy and sample complexity. To illustrate these benefits we provide both synthetic and real-world experimental results.
\end{abstract}

\section{Introduction}
Suppose that $X \in \mathbb R^{n_1\times n_2}$ is a rank-$r$ matrix with $r$ much smaller than $n_1$ and $n_2$. We observe $X$ through a linear operator $\mathcal A:\mathbb R^{n_1\times n_2} \rightarrow \mathbb R^{m}$,
\begin{equation*}
y = \mathcal A(X), \quad y\in \mathbb R^{m}.
\end{equation*}
In recent years there has been a significant amount of progress in our understanding of how to recover $X$ from observations of this form even when the number of observations $m$ is much less than the number of entries in $X$.  (See \cite{davenportoverview} for an overview of this literature.) When $\mathcal A$ is a set of weighted linear combinations of the entries of $X$, this problem is often referred to as the \emph{matrix sensing} problem. In the special case where $\mathcal A$ samples a subset of entries of $X$, it is known as the \emph{matrix completion} problem. There are a number of ways to establish recovery guarantee in these settings. Perhaps the most popular approach for theoretical analysis in recent years has focused on the use of nuclear norm minimization as a convex surrogate for the (nonconvex) rank constraint~\cite{agarwal2012noisy,candes2010matrix,candes2011tight,candes2009exact,candes2010power,davenport20141,
klopp2014noisy, negahban2011estimation, recht2010guaranteed,recht2008necessary}. An alternative, however is to aim to directly solve the problem under an exact low-rank constraint. This leads a non-convex optimization problem, but has several computational advantages over most approaches to minimizing the nuclear norm and is widely used in large-scale applications (such as recommendation systems)~\cite{koren2009bellkor}. In general, popular algorithms for solving the rank-constrained models -- e.g., alternating minimization and alternating gradient descent -- do not have as strong of convergence or recovery error guarantees due to the non-convexity of the rank constraint. However, there has been significant progress on this front in recent years~\cite{Hardt2014fast, hardt2014understanding,jain2014fast, jain2013low,keshavan2009matrix,sun2015guaranteed, zhao2015nonconvex}, with many of these algorithms now having guarantees comparable to those for nuclear norm minimization.

Nearly all of this existing work assumes that the underlying low-rank matrix $X$ remains fixed throughout the measurement process.  In many practical applications, this is a tremendous limitation. For example, users' preferences for various items may change (sometimes quite dramatically) over time. Modelling such drift of user's preference has been proposed in the context of both music and movies as a way to achieve higher accuracy in recommendation systems~\cite{dror2012yahoo,koren2010collaborative}. Another example in signal processing is dynamic non-negative matrix factorization for the blind signal separation problem~\cite{mohammadiha2015state}. In these and many other applications, explicitly modelling the dynamic structure in the data has led to superior empirical performance.  However, our theoretical understanding of dynamic low-rank matrix recovery is still very limited.

In this paper we provide the first theoretical results on the dynamic low-rank matrix recovery problem. We determine the sense in which dynamic constraints can help to recover the underlying time-varying low-rank matrix in a particular dynamic model and quantify this impact through recovery error bounds. To describe our approach, we consider a simple example where we have two rank-$r$ matrices $X^1$ and $X^2$. Suppose that we have a set of observations for each of $X^1$ and $X^2$, given by
\begin{equation*}
y^i = \mathcal{A}^i\left(X^i\right),\quad i = 1,2.
\end{equation*}
The na\"{i}ve approach is to use $y^1$ to recover $X^1$ and $y^2$ to recover $X^2$ separately. In this case the number of observations required to guarantee successful recovery is roughly $m^i \ge C^i r \max(n_1, n_2)$ for $i =1,2$ respectively, where $C^1,C^2$ are fixed positive constants (see \cite{candes2011tight}). However, if we know that $X^2$ is close to $X^1$ in some sense (for example, if $X^2$ is a small perturbation of $X^1$), then the above approach is suboptimal both in terms of recovery accuracy and sample complexity, since in this setting $y^1$ actually contains information about $X^2$ (and similarly, $y^2$ contains information about $X^1$). There are a variety of possible approaches to incorporating this additional information. The approach we will take is inspired by the LOWESS (locally weighted scatterplot smoothing) approach from non-parametric regression. In the case of this simple example, if we look just at the problem of estimating $X^2$, our approach reduces to solving a problem of the form
\[
\min_{X^2} \|\mathcal{A}^2(X^2) - y^2 \|_2^2 + \lambda \|\mathcal{A}^1(X^2) - y^1 \|_2^2 \qquad \text{s.t.} \quad \textrm{rank}\left( X^2\right) \le r,
\]
where $\lambda$ is a parameter that determines how strictly we are enforcing the dynamic constraint (if $X^1$ is very close to $X^2$ we can set $\lambda$ to be larger, but if $X^1$ is far from $X^2$ we will set it to be comparatively small). This approach generalizes naturally to the {\em locally weighted matrix smoothing} (LOWEMS) program described in Section~\ref{sec:prob}. Note that it has a (simple) convex objective function, but a non-convex rank constraint. Our analysis in Section~\ref{sec:theory} shows that the proposed program outperforms the above na\"{i}ve recovery strategy both in terms of recovery accuracy and sample complexity.

We should emphasize that the proposed LOWEMS program is non-convex due to the exact low-rank constraint. Inspired by previous work on matrix factorization, we propose using an efficient alternating minimization algorithm (described in more detail in Section~\ref{sec:alg}). We explicitly enforce the low-rank constraint by optimizing over a rank-$r$ factorization and alternately minimize with respect to one of the factors while holding the other one fixed. This approach is popular in practice since it is typically less computationally complex than nuclear norm minimization based algorithms. In addition, thanks to recent work on global convergence guarantees for alternating minimization for low-rank matrix recovery~\cite{hardt2014understanding, jain2013low,zhao2015nonconvex}, one can reasonably expect similar convergence guarantees to hold for alternating minimization in the context of LOWEMS, although we leave the pursuit of such guarantees for future work.

To empirically verify our analysis, we perform both synthetic and real world experiments, described in Section~\ref{sec:exp}. The synthetic  experimental results demonstrate that LOWEMS outperforms the na\"{i}ve approach in practice both in terms of recovery accuracy and sample complexity. We also demonstrate the effectiveness of LOWEMS in the context of recommendation systems.

Before proceeding, we briefly state some of the notation that we will use throughout. For a vector $x \in \mathbb R^n$, we let $\norm{x}_{p}$ denote the standard $\ell_p$ norm. Given a matrix $X \in \mathbb R ^{n_1\times n_2}$, we use $X_{i:}$ to denote the $i^{\text{th}}$ row of $X$ and $X_{:j}$ to denote the $j^{\text{th}}$ column of $X$. We let $\norm{X}_F$ denote the the Frobenius norm, $\norm{X}_2$ the operator norm, $\norm{X}_*$ the nuclear norm, and $\norm{X}_\infty = \max_{i,j} |X_{ij}|$ the element-wise infinity norm.  Given a pair of matrices $ X,Y \in \mathbb R^{n_1 \times n_2}$, we let $\left\langle X, Y\right\rangle = \sum_{i,j} X_{ij} Y_{ij} = \textrm {Tr}\left( Y^T X \right)$ denote the standard inner product. Finally, we let $n_{\max}$ and $n_{\min}$ denote $\max\{n_1, n_2\}$ and $\min \{n_1, n_2\}$ respectively.

\section{Problem formulation}
\label{sec:prob}

The underlying assumption throughout this paper is that our low-rank matrix is changing over time during the measurement process. For simplicity we will model this through the following discrete dynamic process: at time $t$, we have a low-rank matrix $X^t\in \mathbb R^{n_1\times n_2}$ with rank $r$, which we assume is related to the matrix at previous time-steps via
\[
X^t = f(X^1,\ldots, X^{t-1}) + \epsilon^t,
\]
where $\epsilon^t$ represents noise. We assume that we observe each $X^t$ through a linear operator $\mathcal A^t : \mathbb R^{n_1\times n_2} \rightarrow \mathbb R^{m^t}$,
\begin{equation}
\label{eq:measure}
y^t = \mathcal A^t(X^t)+z^t, \quad y^t,z^t\in \mathbb R^{m^t} ,
\end{equation}
where $z^t$ is measurement noise. In our problem we will suppose that we observe up to $d$ time steps, and our goal is to recover $\{X^t\}_{t=1}^{d}$ jointly from $\{y^t\}_{t=1}^{d}$.

The above model is sufficiently flexible to incorporate a wide variety of dynamics, but we will make several simplifications. First, we note that we can impose the low-rank constraint explicitly by factorizing $X^t$ as $X^t = U^t \left(V^t\right)^T, U^t\in \mathbb R^{n_1\times r}, V^t \in \mathbb R^{n_2 \times r}$. 
In general both $U^t$ and $V^t$ may be changing over time. However, in some applications, it is reasonable to assume that only one set of factors is changing. For example, in a recommendation system where our matrix represent user preferences, if the rows correspond to items and the columns correspond to users, then $U^t$ contains the latent properties of the items and $V^t$ models the latent preferences of the users. In this context it is reasonable to assume that only $V^t$ changes over time~\cite{dror2012yahoo, koren2010collaborative}, and that there is a fixed matrix $U$ (which we may assume to be orthonormal) such that we can write $X^t = U V^t$ for all $t$. Similar arguments can be made in a variety of other applications, including personalized learning systems, blind signal separation, and more. 

Second, we will assume a Markov property on $f$, so that $X^t$ (or equivalently, $V^t$) only depends on the previous $X^{t-1}$ (or $V^{t-1}$). Furthermore, although other dynamic models could be accommodated, for the sake of simplicity in our analysis we consider the simple model on $V^t$ where
\begin{equation}
\label{eq:V-purterbation}
V^t = V^{t-1} + \epsilon^t, \quad t = 2,\ldots,d.
\end{equation}
We will also assume that both $\epsilon^t$ and the measurement noise $z^t$ are i.i.d. zero-mean Gaussian noise. 

To simplify our discussion, we will assume that our goal is to recover the matrix at the most recent time-step, i.e., we wish to estimate $X^d$ from $\{ y^t\}_{t=1}^d$. Our general approach can be stated as follows.  Let $\mathbb C(r) = \{X\in \mathbb R^{n_1\times n_2}: \textrm{rank}(X)\le r\} $. The LOWEMS estimator is given by the following optimization program:
\begin{equation}
\label{eq:weighted-matrix-recovery}
\hat X^d = \mathop{\rm arg~min}_{X\in \mathbb C(r)} \mathcal L\left(X\right) = \mathop{\rm arg~min}_{X\in \mathbb C(r)}  \frac{1}{2}\sum_{t=1}^{d} w_t \norm {\mathcal A^t\left( X \right) - y^t }_2^2,
\end{equation}
where $\{w_t\}_{t=1}^d$ are non-negative weights, and we assume $\sum_{t=1}^d w_t = 1$ to avoid ambiguity. In the following section we provide bounds on the performance of the LOWEMS estimator for two common choices of operators $\mathcal{A}^t$.

\section{Recovery error bounds}
\label{sec:theory}

Given the estimator $\hat{X}^d$ from~\eqref{eq:weighted-matrix-recovery}, we define the recovery error to be $\Delta^d \defeq \hat{X}^d - X^d$. Our goal in this section will be to provide bounds on $\|\hat{X}^d - X^d\|_F$ under two common observation models. Our analysis builds on the following (deterministic) inequality.

\begin{proposition}
\label{thm:basic-ineq}
Both the estimator $\hat{X}^d$ by~\eqref{eq:weighted-matrix-recovery} and~\eqref{eq:weighted-matrix-recovery-completion} satisfies
\begin{equation}
\label{eq:basic-ineq-formal}
\sum_{t=1}^{d} w_t\norm{  \mathcal A^t\left( \Delta^d\right) }_2^2 \le 2\sqrt{2r}\norm{ \sum_{t=1}^{d} w_t \mathcal A^{t*} \left(h^t - z^t \right)}_2 \norm {\Delta^d  }_F,
\end{equation}
where $h^t = \mathcal A^t\left( X^d - X^t\right)$ and $\mathcal{A}^{t*}$ is the adjoint operator of $\mathcal{A}^t$.
\end{proposition}

This is a deterministic result that holds for any set of $\{\mathcal{A}^t\}$.  The remaining work is to lower bound the LHS of \eqref{eq:basic-ineq-formal}, and upper bound the RHS of \eqref{eq:basic-ineq-formal} for concrete choices of $\{\mathcal{A}^t\}$. In the following sections we derive such bounds in the settings of both Gaussian matrix sensing and matrix completion. For simplicity and without loss of generality, we will assume $m^1=\ldots= m^d \eqdef m_0$, so that the total number of observations is simply $m= dm_0$.

\subsection{Matrix sensing setting}
For the matrix sensing problem, we will consider the case where all operators $\mathcal A^t$ correspond to Gaussian measurement ensembles, defined as follows.
\begin{defn}\cite{candes2011tight}
\label{def:gaussian-ensemble}
A linear operator $\mathcal A:\mathbb R^{n_1\times n_2}\rightarrow\mathbb R^{m}$ is a Gaussian measurement ensemble if we can express each entry of $\mathcal A\left( X\right)$ as $\left[ \mathcal A\left( X\right)\right]_i = \left\langle A_i, X\right\rangle$ for a matrix $A_i$ whose entries are i.i.d.\ according to $\mathcal N\left( 0,1/m\right)$, and where the matrices $A_1, \ldots, A_m$ are independent from each other.
\end{defn}
Also, we define the matrix restricted isometry property (RIP) for a linear map $\mathcal A$.
\begin{defn}\cite{candes2011tight}
\label{def:matrix-rip}
For each integer $r=1,\ldots,n_{\min}$, the isometry constant $\delta_r$ of $\mathcal A$ is the smallest  quantity such that
\begin{equation*}
\label{eq:matrix-RIP}
\left( 1- \delta_r \right) \norm{X}_F^2 \le \norm{\mathcal A\left( X\right) }_2^2 \le \left( 1 +  \delta_r \right) \norm{X}_F^2
\end{equation*}
holds for all matrices $X$ of rank at most $r$.
\end{defn}
An important result (that we use in the proof of Theorem~\ref{thm:matrix-sensing}) is that Gaussian measurement ensembles satisfy the matrix RIP with high probability provided $m \ge C r n_{\max}$.  See, for example,~\cite{candes2011tight} for details.

To obtain an error bound in the matrix sensing case we lower bound the LHS of \eqref{eq:basic-ineq-formal} using the matrix RIP and upper bound the stochastic error (the RHS of \eqref{eq:basic-ineq-formal}) using a covering argument. The following is our main result in the context of matrix setting.
\begin{theorem}
\label{thm:matrix-sensing}
Suppose that we are given measurements as in~\eqref{eq:measure} where all $\mathcal A^t$'s are Gaussian measurement ensembles. Assume that $X^t$ evolves according to~\eqref{eq:V-purterbation} and has rank $r$. Further assume that the measurement noise $z^t$ is i.i.d.\ $\mathcal N\left( 0, \sigma_1^2\right)$ for $1\le t\le d$ and that the perturbation noise $\epsilon^t$ is i.i.d.\ $\mathcal N\left(0,\sigma_2^2\right)$ for $2\le t\le d$. If 
\begin{equation}
m_0 \ge D_1\max\left\{  n_{\max}r\sum_{t=1}^d w_t^2,  n_{\max}  \right\}, 
\end{equation}
where $D_1$ is a fixed positive  constant, then the estimator $\hat{X}^d$ from~\eqref{eq:weighted-matrix-recovery} satisfies
\begin{equation}
\label{eq:sensing-results}
\norm{\Delta^d}_F^2 \le  C_0  \left( \sum_{t=1}^d w_t^2 \sigma_1^2 + \sum_{t = 1}^{d-1}    {\left(d - t\right)} w^2_t \sigma_2^2 \right) n_{\max} r
\end{equation}
with probability at least $P_1 = 1- dC_1\exp\left(-c_1 n_2 \right) $, where $C_0, C_1, c_1$ are positive constants.
\end{theorem}

If we choose the weights as $w_d=1$ and $w_t=0$ for $1\le t\le d-1$, the bound in Theorem \ref{thm:matrix-sensing} reduces to a bound matching classical (static) matrix recovery results (see, for example, \cite{candes2011tight} Theorem 2.4). Also note that in this case Theorem \ref{thm:matrix-sensing} implies exact recovery when the sample complexity is $O(rn/d)$. In order to help interpret this result for other choices of the weights, we note that for a given set of parameters, we can determine the optimal weights that will minimize this bound.  Towards this end, we define $\kappa := {\sigma_2^2} /{\sigma_1^2}$ and set $p_t = {\left(d - t\right)},1\le t\le d$. Then one can calculate the optimal weights by solving the following quadratic program:
\begin{equation}
\label{eq:optimal_weights}
\left\{w^*_t\right\}_{t=1}^d  =  \mathop{\rm arg~min}_{\sum_t w_t = 1;\ w_t\ge 0  } \sum_{t=1}^d w_t^2 + \sum_{t = 1}^{d-1} p_t \kappa w^2_t .
\end{equation}
Using the method of Lagrange multipliers one can show that~\eqref{eq:optimal_weights} has the analytical solution:
\begin{equation}
\label{eq:optimal_weights_analytical}
w_j^* = \frac{1}{\sum_{i=1}^{d} \frac{1}{1+p_i\kappa}} \frac{1}{1+p_j\kappa},\quad 1\le j\le d.
\end{equation}

A simple special case occurs when $\sigma_2 = 0$.  In this case all $V^t$'s are the same, and the optimal weights go to $w^t = \frac{1}{d}$ for all $t$.  In contrast, when $\sigma_2$ grows large the weights eventually converge to $w_d = 1$ and $w^t = 0$ for all $t \neq d$.  This results in essentially using only $y^d$ to recover $X^d$ and ignoring the rest of the measurements. Combining these, we note that when the $\sigma_2$ is small, we can gain by a factor of approximately $d$ over the na\"{i}ve strategy that ignores dynamics and tries to recover $X^d$ using only $y^d$.  Notice also that the minimum sample complexity is proportional to $\sum_{t=1}^d w_t^2$ when $r/d$ is relatively large.  Thus, when $\sigma_2$ is small, the required number of measurements can be reduced by a factor of $d$ compared to what would be required to recover $X^d$ using only $y^d$.

\subsection{Matrix completion setting}
For the matrix completion problem, we consider the following simple uniform sampling ensemble:
\begin{defn}
A linear operator $\mathcal A:\mathbb R^{n_1\times n_2} \rightarrow \mathbb R^{m}$ is a uniform sampling ensemble (with replacement) if all sensing matrices $A_i$ are i.i.d.\ uniformly distributed on the set
\begin{equation*}
\mathcal X = \left\{ e_j\left(n_1\right) e_k^T\left(n_2\right), 1\le j\le n_1, 1\le k \le n_2 \right),
\end{equation*}
where $e_j\left( n\right)$ are the canonical basis vectors in $\mathbb R^n$. We let $p =  m_0/(n_1n_2)$ denote the fraction of sampled entries.
\end{defn}

For this observation architecture, our analysis is complicated by the fact that it does not satisfy the matrix RIP. (A quick problematic example is a rank-$1$ matrix with only one non-zero entry.) To handle this we follow the typical approach and restrict our focus to matrices that satisfy certain {\em incoherence} properties.
\begin{defn} (Subspace incoherence \cite{hardt2014understanding}) Let $U\in\mathbb R^{n\times r}$ be the orthonormal basis for an $r$-dimensional subspace $\mathcal U$, then the incoherence of $\mathcal U$ is defined as $\mu(\mathcal U) \defeq \max_{i\in [n]}\frac{\sqrt n}{\sqrt r} \norm{ e_i^T U}_2$, where $e_i$ denotes the $i^{\text{th}}$ standard basis vector. We also simply denote $\mu( \textrm{span}(U))$ as $\mu(U)$.
\end{defn}
\begin{defn} (Matrix incoherence \cite{jain2013low}) A rank-$r$ matrix $X\in \mathbb R^{n_1\times n_2}$ with SVD $X = U \Sigma V^T$ is incoherent with parameter $\mu$ if
\begin{equation*}
\norm{U_{:i}}_2 \le \frac{\mu \sqrt{r}}{\sqrt{n_1}} \quad \textrm{for any } i\in [n_1]\quad \textrm{and} \quad \norm{V_{:j}}_2 \le \frac{\mu \sqrt{r}}{\sqrt{n_2}} \quad \textrm{for any } j\in [n_2],
\end{equation*}
i.e., the subspaces spanned by the columns of $U$ and $V$ are both $\mu$-incoherent.
\end{defn}

The incoherence assumption guarantees that $X$ is far from sparse, which make it possible to recover $X$ from incomplete measurements since a measurement contains roughly the same amount of information for all dimensions.

To proceed we also assume that the matrix $X^d$ has ``bounded spikiness'' in that the maximum entry of $X^d$ is bounded by $a$, i.e.,$\norm{X^d}_\infty\le a$. To exploit the spikiness constraint below we replace the optimization constraints $\mathbb C \left( r \right)$ in \eqref{eq:weighted-matrix-recovery} with $ \mathbb C\left( r, a\right)\defeq = \left\{ X\in \mathbb R^{n_1\times n_2}: \textrm{rank}\left(X \right) \le r,\norm{X}_\infty \le a \right\} $:
\begin{equation}
\label{eq:weighted-matrix-recovery-completion}
\hat X^d =  \mathop{\rm arg~min}_{X\in \mathbb C(r,a)} \mathcal L\left(X\right) = \mathop{\rm arg~min}_{X\in \mathbb C(r,a)}  \frac{1}{2}\sum_{t=1}^{d} w_t \norm {\mathcal A^t\left( X \right) - y^t }_2^2.
\end{equation}
Note that Proposition~\ref{thm:basic-ineq} still holds for~\eqref{eq:weighted-matrix-recovery-completion}.

To obtain an error bound in the matrix completion case, we lower bound the LHS of \ref{eq:basic-ineq-formal} using a restricted convexity argument (see, for example,~\cite{negahban2012restricted}) and upper bound the RHS using matrix Bernstein inequality. The result of this approach is the following theorem.
\begin{theorem}
\label{thm:matrix-completion}
Suppose that we are given measurements as in~\eqref{eq:measure} where all $\mathcal A^t$'s are uniform sampling ensembles. Assume that $X^t$ evolves according to~\eqref{eq:V-purterbation}, has rank $r$, and is incoherent with parameter $\mu_0$ and  $\norm{X^d}_\infty \le a$. Further assume that the perturbation noise and the measurement noise satisfy the same assumptions in Theorem~\ref{thm:matrix-sensing}. If 
\begin{equation}
m_0 \ge D_2  n_{\min}\log^2 (n_1 + n_2)\phi'(w), 
\end{equation}
where $\phi'(w) =  \frac{ \max_{t} w_t^2 \left( (d-t)\mu_0^2 r\sigma_2^2/n_1  + \sigma_1^2\right)} {\sum_{t=1}^d w_t^2 \left( (d-t)\sigma_2^2  + \sigma_1^2 \right) }$, then the estimator $\hat{X}^d$ from~\eqref{eq:weighted-matrix-recovery-completion} satisfies
\begin{equation}
\label{eq:bound-completion}
\norm{\Delta^d}_F^2 \le \max\left\{  B_1 \defeq C_2 a^2 n_1n_2\sqrt{\frac{\sum_{t=1}^d w_t^2\log(n_1+n_2)}{m_0}}, B_2 \right\},
\end{equation}
with probability at least $P_1 = 1-5/(n_1+n_2) - 5dn_{\max} \exp(-n_{\min})$, where
\begin{equation}
\label{eq:B2}
B_2 = \frac{C_3 rn_1^2n_2^2\log(n_1+n_2)}{n_{\min}m_0} \left (  \left( \sum_{t=1}^d  w_t^2 \sigma_1^2 + \sum_{t=1}^{d-1} (d-t) w_t^2 \sigma_2^2\right)   +  \sum_{t=1}^d w_t^2 a^2\right),
\end{equation}
and $C_2, C_3, D_2$ are absolute positive constants.
\end{theorem}
If we choose the weights as $w_d=1$ and $w_t=0$ for $1\le t\le d-1$, the bound in Theorem \ref{thm:matrix-completion} reduces to a result comparable to classical (static) matrix completion results (see, for example, \cite{klopp2014noisy} Theorem 7).
Moreover, from the $B_2$ term in~\eqref{eq:bound-completion}, we obtain the same dependence on $m$  as that of  \eqref{eq:sensing-results}, i.e., $1/m$.
However, there are also a few key differences between Theorem~\ref{thm:matrix-sensing} and our results for matrix completion. In general the bound is loose in several aspects compared to the matrix sensing bound. For example, when $m_0$ is small, $B_1$ actually dominates, in which case the dependence on $m$ is actually $1/\sqrt{m}$ instead of $1/m$. When $m_0$ is sufficiently large, then $B_2$ dominates, in which case we can consider two cases. The first case corresponds to  when $a$ is relatively large compared to $\sigma_1,\sigma_2$ -- i.e., the low-rank matrix is spiky. In this case the term containing $a^2$ in $B_2$ dominates, and the optimal weights are equal weights of $1/d$. This occurs because the term involving $a$ dominates and there is little improvement to be obtained by exploiting temporal dynamics. In the second case, when $a$ is relatively small compared to $\sigma_1, \sigma_2$ (which is usually the case in practice), the bound can be simplified to
\begin{equation*}
\norm{\Delta}_F^2 \le \frac{c_3 rn_1^2n_2^2\log(n_1+n_2)}{n_{\min}m_0} \left (  \left( \sum_{t=1}^d  w_t^2 \sigma_1^2 + \sum_{t=1}^{d-1} (d-t) w_t^2  \sigma_2^2\right)\right).
\end{equation*}
The above bound is much more similar to the bound in \eqref{eq:sensing-results} from Theorem \ref{thm:matrix-sensing}. In fact, we can also obtain the optimal weights by solving the same quadratic program as \eqref{eq:optimal_weights}.

When $n_1\approx n_2$, the sample complexity is $\Theta(n_{\min}\log^2 (n_1 + n_2)\phi'(w))$. In this case Theorem \ref{thm:matrix-completion} also implies a similar sample complexity reduction as we observed in the matrix sensing setting. However, the precise relations between sample complexity and weights $w_t$'s are different in these two cases (deriving from the fact that the proof uses matrix Bernstein inequalities in the matrix completion setting rather than concentration inequalities of Chi-squared variables as in the matrix sensing setting).

\section{An algorithm based on alternating minimization}
\label{sec:alg}
As noted in Section~\ref{sec:prob}, any rank-$r$ matrix can be factorized as $X = UV^T$ where $U$ is $n_1 \times r$ and $V$ is $n_2 \times r$, therefore the LOWEMS estimator in~\eqref{eq:weighted-matrix-recovery} can be reformulated as
\begin{equation}
\label{eq:factorized-rec}
\hat X^d =  \mathop{\rm arg~min}_{X\in\mathcal C(r)} \mathcal L\left(X\right) = \mathop{\rm arg~min}_{X=UV^T} \sum_{t=1}^{d} \frac{1}{2}w_t \norm {\mathcal A^t\left( UV^T \right) - y^t }_2^2.
\end{equation}

The above program can be solved by alternating minimization (see \cite{koren2010collaborative}), which alternatively minimizes the objective function over $U$ (or $V$) while holding $V$ (or $U$) fixed until a stopping criterion is reached. Since the objective function is quadratic, each step in this procedure reduces to conventional weighted least squares, which can be solved via efficient numerical procedures. Theoretical guarantees for global convergence of alternating minimization for the static matrix sensing/completion problem have recently been established in~\cite{hardt2014understanding, jain2013low,zhao2015nonconvex} by treating the alternating minimization as a noisy version of the power method. Extending these results to establish convergence guarantees for~\eqref{eq:factorized-rec} would involve analyzing a weighted power method. We leave this analysis for future work, but expect that similar convergence guarantees should be possible in this setting.

\section{Simulations and experiments}
\label{sec:exp}
\subsection{Synthetic simulations}
Our synthetic simulations consider both matrix sensing and matrix completion, but with an emphasis on matrix completion. We set $n_1 = 100$, $n_2 = 50$, $d = 4$ and $r=5$. We consider two baselines: {\bf baseline one} is only using $y^d$ to recover $X^d$ and simply ignoring $y^1, \ldots y^{d-1}$; {\bf baseline two} is using $\{y^t\}_{t=1}^d$ with equal weights.  Note that both of these can be viewed as special cases of LOWEMS with weights $(0, \ldots, 0, 1)$ and $(\frac{1}{d}, \frac{1}{d}, \ldots, \frac{1}{d})$ respectively. Recalling the formula for the optimal choice of weights in~\eqref{eq:optimal_weights_analytical}, it is easy to show that baseline one is equivalent to the case where $\kappa  = ( \sigma_2^2 )/(\sigma_1^2) \rightarrow \infty$ and the baseline two equivalent to the case where $\kappa \rightarrow 0$. This also makes intuitive sense since $\kappa \rightarrow \infty$ means the perturbation is arbitrarily large between time steps, while $\kappa \rightarrow 0$ reduces to the static setting.

\emph{1). Recovery error.} \
In this simulation, we set $m_0 = 4000$ 
and set the measurement noise level $\sigma_1$ to $0.05$. We vary the perturbation noise level $\sigma_2$. For every pair of $\left( \sigma_1, \sigma_2\right)$ we perform $10$ trials, and show the average relative recovery error ${\norm{\Delta^d}_F^2} / {\norm{X^d}_F^2}$. Figure \ref{fig:rec_err} illustrates how LOWEMS reduces the recovery error compared to our baselines. As one can see, when $\sigma_2$ is small, the optimal $\kappa$, i.e., $\sigma_2^2 /\sigma_1^2$, generates nearly equal weights (baseline two), reducing recovery error approximately by a factor of $4$ over baseline one, which is roughly equal to $d$ as expected. As $\sigma_2$ grows, the recovery error of baseline two will increase dramatically due to the perturbation noise. However in this case the optimal $\kappa$ of LOWEMS grows with it, leading to a more uneven weighting and to somewhat diminished performance gains. We also note that, as expected, LOWEMS converges to baseline one when $\sigma_2$ is large.

\begin{figure}[t]
	\centering

	\begin{subfigure}[c]{0.03\textwidth}
		\caption{}
	\end{subfigure}%
	\begin{minipage}[c]{0.47\textwidth}
		\includegraphics[width=1\linewidth]{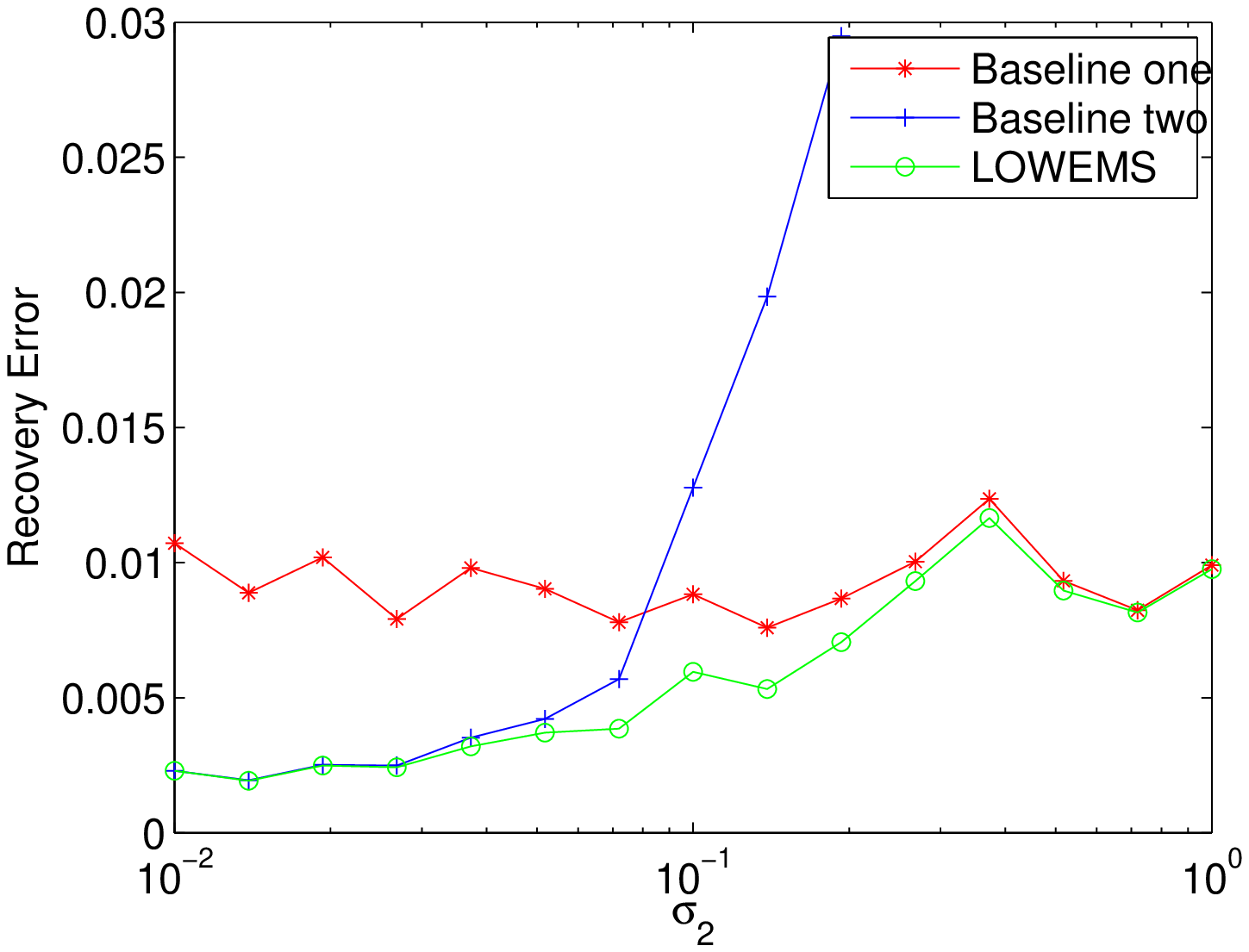}
	\end{minipage}%
	\begin{subfigure}[c]{0.03\textwidth}
		\caption{}
	\end{subfigure}%
	\begin{minipage}[c]{0.47\textwidth}
		\includegraphics[width=1\linewidth]{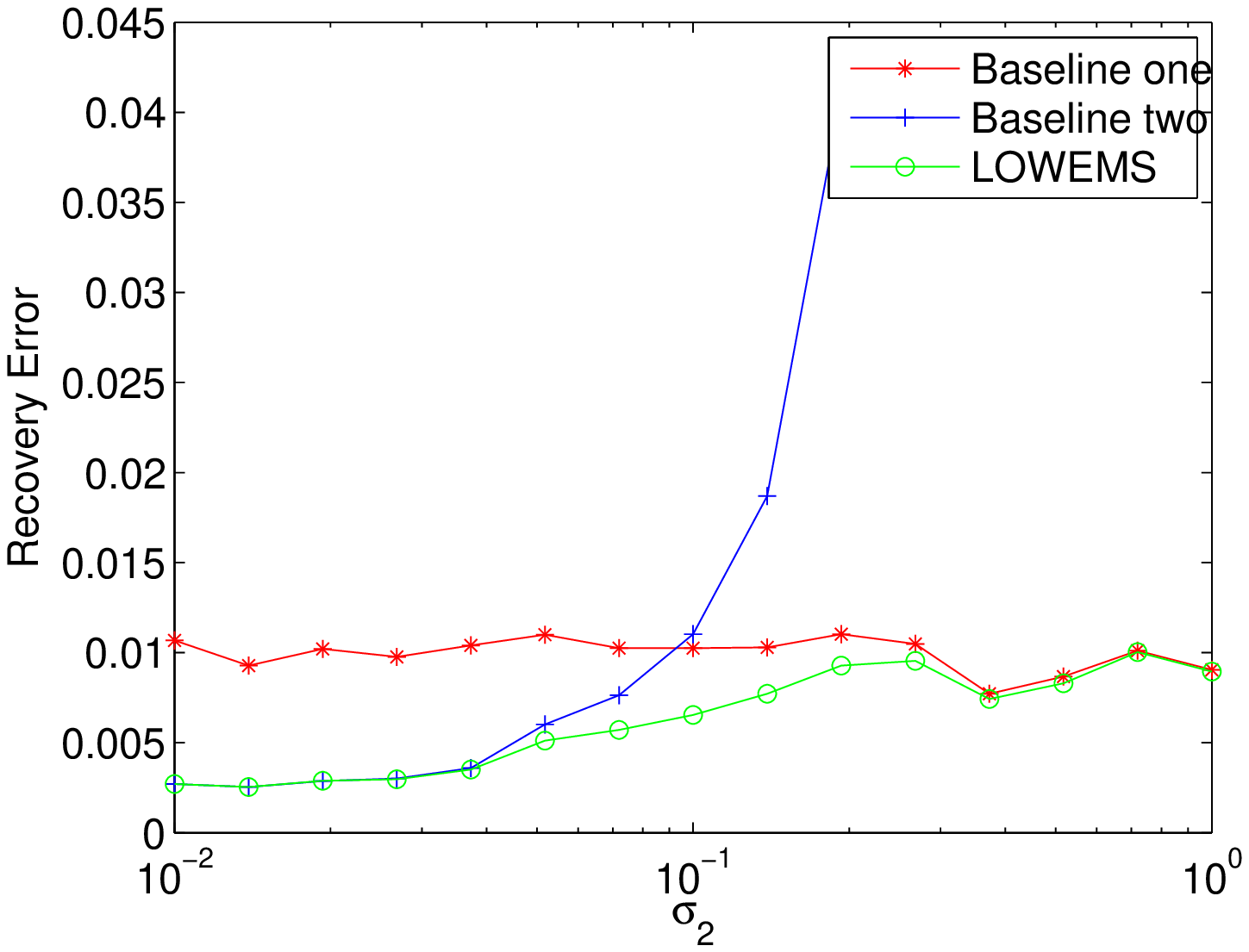}
	\end{minipage}
	\caption{Recovery error under different levels of perturbation noise. (a) matrix sensing. (b) matrix completion.}%
	\label{fig:rec_err}%

\end{figure}
\emph{ 2). Sample complexity.} \
In the interest of conciseness we only provide results here for the matrix completion setting (matrix sensing yields broadly similar results). In this simulation we vary the fraction of observed entries $p$ to empirically find the minimum sample complexity required to guarantee successful recovery (defined as a relative error $\le0.04$). We compare the sample complexity of the proposed LOWEMS to baseline one and baseline two under different perturbation noise level $\sigma_2$. For fixed $\sigma_2$, the relative recovery error is the averaged over $10$ trials. Figure \ref{fig:sample} illustrates how LOWEMS reduces the sample complexity required to guarantee successful recovery. When the perturbation noise is weaker than the measurement noise, the sample complexity can be reduced approximately by a factor of $d$ compared to baseline one. When the perturbation noise is much stronger than measurement noise, the recovery error of baseline two will increase due to the perturbation noise and hence the sample complexity increase rapidly. However in this case proposed LOWEMS still achieves relatively small sample complexity and its sample complexity converges to baseline one when $\sigma_2$ is relatively large.

\begin{figure}[tbp]
    \centering
\includegraphics[width=0.48\linewidth]{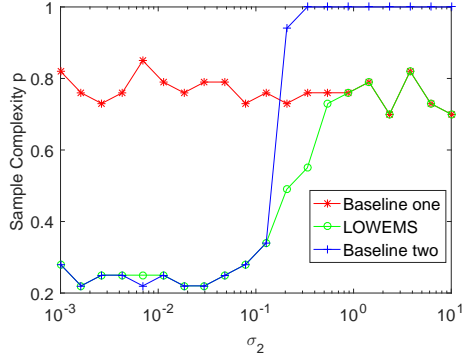} %
    \caption{Sample complexity under different levels of perturbation noise (matrix completion). }%
    \label{fig:sample}%
\end{figure}

\subsection{Real world experiments}
We next test the LOWEMS approach in the context of a recommendation system using the (truncated) Netflix dataset. We eliminate those movies with few ratings, and those users rating few movies, and generate a truncated dataset with $3199$ users, $1042$ movies, $2462840$ ratings, and hence the fraction of visible entries in the rating matrix is $\approx 0.74$. All the ratings are distributed over a period of $2191$ days.
For the sake of robustness, we additionally impose a Frobenius norm penalty on the factor matrices $U$ and $V$ in~\eqref{eq:factorized-rec}. We keep the latest (in time) $10\%$ of the ratings as a testing set. The remaining ratings are split into a validation set and a training set for the purpose of cross validation. We divide the remaining ratings into $d \in \{1, 3, 6, 8\}$ bins respectively with same time period according to their timestamps. We use $5$-fold cross validation, and we keep $1/5$ of the ratings from the $d^{\text{th}}$ bin as a validation set. The number of latent factors $r$ is set to $10$. The Frobenius norm regularization parameter $\gamma$ is set to $1$. We also note that in practice one likely has no prior information on $\sigma_1$, $\sigma_2$ and hence $\kappa$. However, we use model selection techniques like cross validation to select the best $\kappa$ incorporating the unknown prior information on measurement/perturbation noise.
We use root mean squared error (RMSE) to measure prediction accuracy. Since alternating minimization uses a random initialization, we generate $10$ test RMSE's (using a boxplot) for the same testing set. Figure~\ref{fig:netflix}(a) shows that the proposed LOWEMS estimator improves the testing RMSE significantly with appropriate $\kappa$. Additionally, the performance improvement increases as $d$ gets larger.

To further investigate how the parameter $\kappa$ affects accuracy, we also show the validation RMSE compared to $\kappa$ in Figure~\ref{fig:netflix}(b). When $\kappa \approx 1$,  LOWEMS achieves the best RMSE on the validation data. This further demonstrates that imposing an appropriate dynamic constraint should improve recovery accuracy in practice.

\begin{figure}[t]
	\centering

	\begin{subfigure}[c]{0.03\textwidth}
		\caption{}
	\end{subfigure}%
	\begin{minipage}[c]{0.47\textwidth}
		\includegraphics[width=1\linewidth]{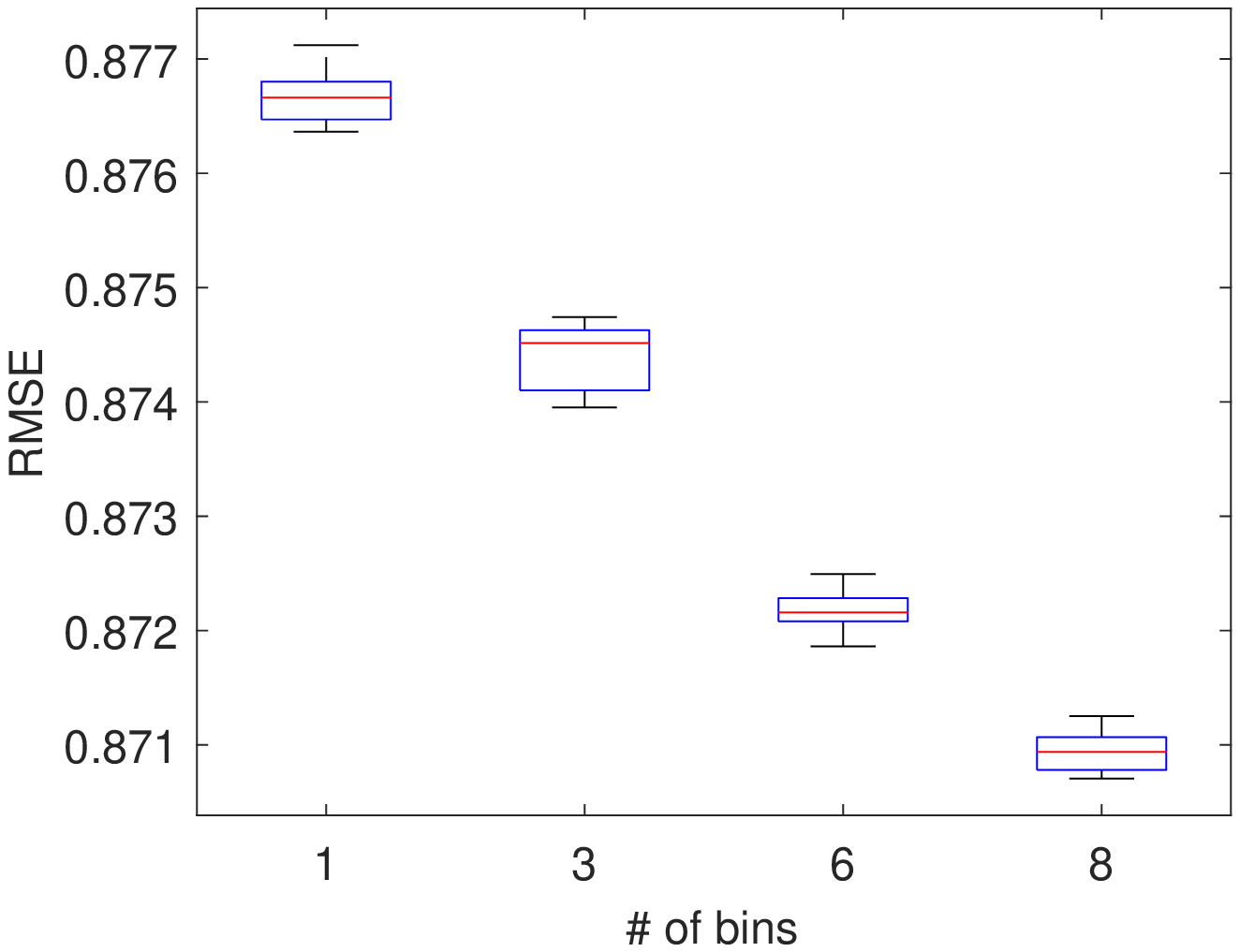}
	\end{minipage}%
	\begin{subfigure}[c]{0.03\textwidth}
		\caption{}
	\end{subfigure}%
	\begin{minipage}[c]{0.47\textwidth}
		\includegraphics[width=1\linewidth]{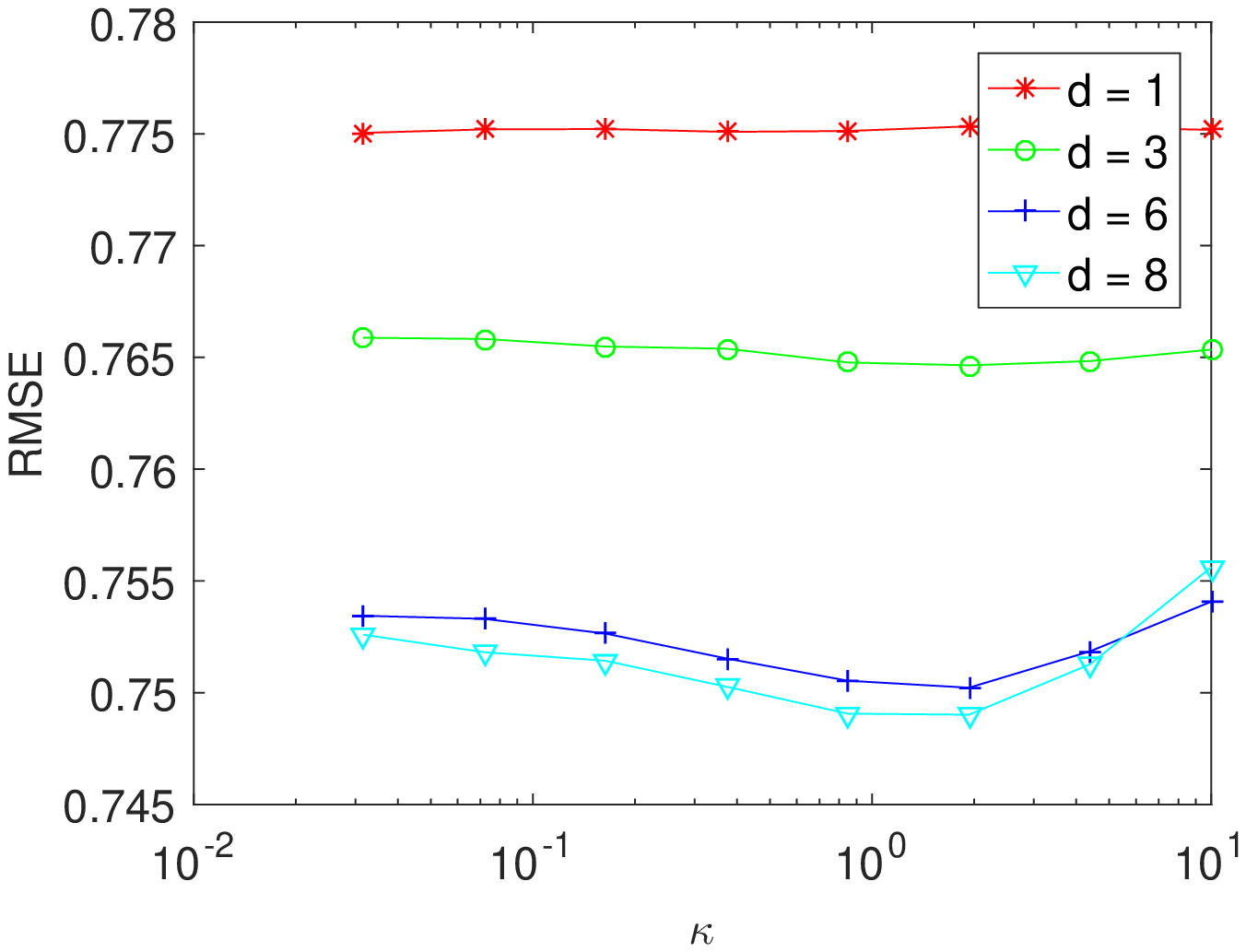}
	\end{minipage}
	\caption{Experimental results on truncated Netflix dataset. (a) Testing RMSE vs.\ number of time steps. (b) 	Validation RMSE vs.\ $\kappa$.}%
	\label{fig:netflix}%

\end{figure}

\section{Conclusion}
In this paper we consider the low-rank matrix recovery problem in a novel setting, where one of the factor matrices changes over time. We propose the locally weighted matrix smoothing (LOWEMS) framework, and have established error bounds for LOWEMS in both the matrix sensing and matrix completion cases. Our analysis quantifies how the proposed estimator improves recovery accuracy and reduces sample complexity compared to static recovery methods. Finally, we provide both synthetic and real world experimental results to verify our analysis and demonstrate superior empirical performance when exploiting dynamic constraints in a recommendation system.

\bibliographystyle{abbrv}
\begin{small}
 \bibliography{refs}
\end{small}

\newpage

\appendix
\section{Proof of Proposition \ref{thm:basic-ineq} }
\begin{proof}
Let $x \defeq \textrm{vec}\left( X\right) \in \mathbb R^{n_1n_2}$ and $\tilde{\mathcal L} \left( x\right)  \defeq \mathcal L\left( X\right)$. Since the objective function is continuous in $X$ and the set $\mathbb C\left(r\right)$ is compact, $\mathcal L\left( X\right)$ achieves a minimizer at some point $\hat{X}^d \in \mathbb C\left(r\right)$.

Since $\hat{X}^d$ is a minimizer of the constrained problem, then for all matrices $X \in \mathbb C \left(r\right)$ we have the following inequality
\begin{equation}
\label{eq:basic-ineq}
\tilde{\mathcal L}\left( \hat{x}^d \right) - \tilde{\mathcal L}\left( x \right)\le 0.
\end{equation}
By the second-order Taylor's theorem, we expand $\tilde{\mathcal L}\left(x\right)$ around $x^d= \textrm{vec}\left(X^d\right)$
\begin{equation}
\label{eq:tylor}
\tilde{\mathcal L}\left(x\right) = \tilde{\mathcal L}\left(x^d\right) + \left\langle \nabla \tilde{\mathcal L}\left(x^d\right), x - x^d \right\rangle + \frac{1}{2} \left\langle \nabla^2 \tilde{\mathcal L}\left(\bar x\right)\left( x - x^d\right), x - x^d \right\rangle,
\end{equation}
where $\bar x = \alpha x^d  + \left(1-\alpha \right) x$ for some $\alpha \in \left[0,1\right]$. Plugging \eqref{eq:tylor} with $x = \hat{x}^d$ into \eqref{eq:basic-ineq} we obtain
\begin{equation}
\label{eq:basic-ineq2}
 \left\langle \nabla \tilde{\mathcal L}\left(x^d\right), \hat{x}^d - x^d \right\rangle + \frac{1}{2} \left\langle \nabla^2 \tilde{\mathcal L}\left(\bar x\right)\left( \hat{x}^d - x^d\right), \hat{x}^d - x^d \right\rangle \le 0.
\end{equation}

Through some algebraic manipulation we have the following expression for the gradient of $\tilde{\mathcal L}\left(x\right)$:
\begin{equation}
\label{eq:gradient}
\nabla \tilde{\mathcal L}\left(x\right) =  \textrm{vec}\left( \sum_{t=1}^{d} w_t \mathcal A^{t*} \left[\mathcal A^t\left( X\right) - y^t  \right] \right).
\end{equation}
Based on the above gradient it follows that
\begin{equation}
\label{eq:hessian}
\nabla^2 \tilde{\mathcal L}\left(x\right) b =   \textrm{vec}\left( \sum_{t=1}^{d} w_t \mathcal A^{t*} \left[\mathcal A^t\left( B\right)  \right] \right),
\end{equation}
where $b=\textrm{vec}\left( B\right)$.

Now based on \eqref{eq:gradient} and \eqref{eq:hessian}, the absolute value of first term in \eqref{eq:basic-ineq2} can be bounded as
\begin{equation}
\label{eq:inner-prod}
\begin{split}
\left| \langle \nabla \tilde{\mathcal L}\left(x^d\right), \hat{x}^d - x^d\rangle \right|
& = \left|  \left\langle \sum_{t=1}^{d} w_t \mathcal A^{t*} \left[\mathcal A^t\left( X^d\right) - y^t  \right] ,\Delta^d  \right\rangle \right|\\
&\le \norm{ \sum_{t=1}^{d} w_t \mathcal A^{t*} \left[\mathcal A^t\left( X^d\right) - y^t  \right]}_2 \norm {\Delta^d  }_* \\
&\le \norm{ \sum_{t=1}^{d} w_t \mathcal A^{t*} \left(h^t - z^t \right)}_2 \sqrt{2r}\norm {\Delta^d  }_F \\
\end{split}
\end{equation}
The first inequality above used the trace dual norm inequality, while the second inequality follows from a basic inequality for rank-$2r$ matrices. Similarly the second term in \eqref{eq:basic-ineq2} is
\begin{equation}
\label{eq:second-order}
\begin{split}
\frac{1}{2} \left\langle \nabla^2 \tilde{\mathcal L}\left(\bar x\right)\left( \hat{x}^d - x^d\right), \hat{x}^d - x^d \right\rangle &= \frac{1}{2}\left\langle \sum_{t=1}^{d} w_t \mathcal A^{t*} \mathcal A^t\left( \Delta^d\right)   ,\Delta^d  \right\rangle \\
& =  \frac{1}{2} \sum_{t=1}^{d} w_t \left\langle  \mathcal A^t\left( \Delta^d\right)  , \mathcal A^{t} \left(\Delta^d  \right) \right\rangle.
\end{split}
\end{equation}
The result follows from combining \eqref{eq:inner-prod} and \eqref{eq:second-order}. Note that the above proof holds if we replace $\mathbb  C \left( r,\right)$ with $\mathbb C\left( r, a \right)$, which completes our proof.
\end{proof}
\section{Proof of Theorem \ref{thm:matrix-sensing} }
\begin{proof}
The proof consists of lower bounding the LHS  of \eqref{eq:basic-ineq-formal} and upper bounding the RHS of \eqref{eq:basic-ineq-formal}.

We use the following lemma to lower bound $\sum_{t=1}^d w_t \norm{\mathcal A^t\left( \Delta^d\right)}_2^2$.
\begin{lemma}
\label{lemma:matrix-RIP}
Suppose the linear operator $\mathcal A^t:\mathbb R^{n_1\times n_2}\rightarrow \mathbb R^{m_0}$ is random Gaussian ensemble for all $ 1\le t\le d$. If $m_0>Dn_{\max}r\sum_{t=1}^d w_t^2$, the composite operator $\left\{ \sqrt{w_t}\mathcal A^t\right\}_{t=1}^d$ satisfies the rank-$2r$ matrix RIP with constant $\delta_{2r}\le \delta$ with probability exceeding $1-C\exp\left( -cm_0\right)$, where $D, C$ and $c$ (which depends on $\sigma$) are absolute positive constants.
\end{lemma}
\begin{proof}
See Appendix \ref{sec:proof-matrix-rip}.
\end{proof}

Next lemma gives us an upper bound for the stochastic error $\norm{ \sum_{t=1}^{d} w_t \mathcal A^{t*} \left(h^t-z^t \right)}_2 $.
\begin{lemma}
\label{lemma:upper-stoch-sensing}
Under the assumptions of Theorem \ref{thm:matrix-sensing}, when $m_0 \ge Dn_{\max}$, we have
\begin{equation*}
\norm{ \sum_{t=1}^{d} w_t \mathcal A^{t*} \left(h^t-z^t \right)}_2  \le C_1 \sqrt{n_{\max} (1+\delta_1) \left(\sum_{t=1}^d w_t^2 \sigma_1^2 + \sum_{t = 1}^{d-1} {\left(d - t\right)} w^2_t \frac{2rn_2}{m_0} \sigma_2^2 \right)}
\end{equation*}
with probability exceeding $1- dC\exp(-cn_2)$, where $D, C_1, C, c$ are positive constants and $\delta_1$ is the rank-$1$ matrix RIP parameter for all $\mathcal A^t$'s.
\end{lemma}
\begin{proof}
See Appendix \ref{sec:proof-upper-stoch-sensing}.
\end{proof}

Theorem \ref{thm:matrix-sensing} follows by combining Lemma \ref{lemma:matrix-RIP}, Lemma \ref{lemma:upper-stoch-sensing} and Definition \ref{def:matrix-rip}.
\end{proof}

\section{Proof of Lemma \ref{lemma:matrix-RIP}}
\label{sec:proof-matrix-rip}
\begin{proof}
First we introduce the following theorem providing a double-sided tail bound on the sum of independent sub-exponential random variables.
\begin{theorem}
\label{thm:ind-subexp}
For independent $X_i$ sub-exponential with parameters $\left(\sigma_i, b_i \right) $, with mean $\mu_i$,
\begin{equation*}
\mathbb P\left(\left| \sum_{i=1}^n \left( X_i -\mu_i \right) \right| \ge nt\right) \le 2\exp\left( -\frac{nt^2}{2\left(\sigma^2 + bt\right) }\right),
\end{equation*}
where $\sigma^2 = \sum_i\sigma_i^2$ and $b = \max_i b_i$.
\end{theorem}

We now lower bound $\sum_{t=1}^{d} w_t\norm{  \mathcal A^t\left( \Delta^d\right) }_2^2$. Since all $\mathcal A^t$'s are Gaussian random measurement ensembles, then a particular measurement $ \left\langle A^t_i, \Delta^d \right\rangle^2$ is distributed as $m_0^{-1} \norm{\Delta^d}_F^2 \chi^2\left(1\right)$.  Therefore $\sum_{t=1}^{d} w_t\norm{  \mathcal A^t\left( \Delta^d\right) }_2^2 = \sum_{t,i}w_t \left\langle A^t_i,\left( \Delta^d\right) \right\rangle^2$ is a weighted sum of i.i.d.\ $\chi^2\left(1\right)$ random variables. Since $\chi^2\left(1\right)$ is sub-exponential with parameters $\left(4,4\right)$, Theorem \ref{thm:ind-subexp} implies a double-sided tail bound for $\sum_{t=1}^d w_t\norm{ {\mathcal A^t}\left( \Delta^d\right) }_2^2$:  for any given $\Delta^d \in \mathbb R^{n_1\times n_2}$ and any fixed $0<s<1$
\begin{equation*}
\label{eq:operator-concentration}
\mathbb P\left( \left| \sum_{t=1}^d w_t\norm{ {\mathcal A^t}\left( \Delta^d\right) }_2^2 - \norm{\Delta^d}_F^2  \right|     \le s \norm{\Delta^d}_F^2 \right) \le 2\exp\left(-\frac{m_0s^2}{8\sum_{t=1}^d w_t^2+ 8w_{\max} s} \right),
\end{equation*}
where $w_{\max} = \max\{w_1,\ldots,w_d\}$. The probability can be further simplified if $s$ is very small ($\le 1/d$).

Rank of $\Delta^d$ is at most $2r$ since $\hat{X}^d, X^d$ are rank-$r$ matrices. By Theorem 2.3 in \cite{candes2011tight} (one may see the proof if necessary) if $m_0>Dn_{\max}r\sum_{t=1}^d w_t^2$, the composite operator $\left\{ \sqrt{w_t}\mathcal A^t \right\}_{t=1}^d$ satisfies the rank-$2r$ matrix RIP with constant $\delta_{2r}\le \delta$ with probability exceeding $1-C\exp\left( -cm_0\right)$, where $C$ and $c$ (depends on $\delta$) are absolute positive constants.
\end{proof}

\section{Proof of Lemma \ref{lemma:upper-stoch-sensing}}
\label{sec:proof-upper-stoch-sensing}
\begin{proof}
Let $W = \sum_{t=1}^d w_t \mathcal A^{t*}\left( h^t-z^t\right)$ and $n = n_{\max}$ for short. Following the basic framework of the proof of Lemma 1.1 in \cite{candes2011tight}, we use $\epsilon$-nets method to bound the stochastic error $\norm{W}_2 $. The operator norm of $W$ is
\begin{equation*}
\norm{W}_2 = \sup_{\norm{u} = \norm{v} = 1} \left\langle u, W v\right\rangle,
\end{equation*}
Consider a $1/4$-net $\mathcal N_{1/4}$ of the unite sphere $S^{n-1}$ with $\left| \mathcal N_{1/4} \right| \le 12^{n}$ (see (III.1) in \cite{candes2011tight}). For any $v, u \in S^{n-1}$
\begin{equation*}
\begin{split}
\left\langle u, Wv\right\rangle &= \left\langle u - u_0, Wv \right\rangle + \left\langle u_0,W\left(v - v_0\right) \right\rangle + \left\langle u_0, Wv_0\right\rangle \\
&\le \norm{W}_2\norm{u-u_0}_2 + \norm{W}_2\norm{v-v_0}_2 +  \left\langle u_0, Wv_0\right\rangle,
\end{split}
\end{equation*}
for some $v_0, w_0\in \mathcal N_{1/4}$ obeying $\norm{u-u_0}_2 \le 1/4$ and $\norm{v-v_0} \le 1/4$. So the operator norm of $W$ is
\begin{equation*}
\norm{W}_2 \le 2  \sup_{u_0, v_0 \in \mathcal N_{1/4}} \left\langle u_0, Wv_0\right\rangle.
\end{equation*}
For fixed $u_0, v_0$
\begin{equation*}
\begin{split}
 \left\langle u_0, Wv_0\right\rangle = \textrm{Tr}\left(u_0^T W v_0 \right) = \textrm{Tr}\left(v_0 u_0^T W \right) = \left\langle u_0v_0^T, W\right\rangle  = \sum_{t=1}^dw_t \left\langle \mathcal A^t\left( u_0v_0^T \right),h^t-z^t\right\rangle.
 \end{split}
\end{equation*}
Let $Z =  \sum_{t=1}^dw_t \left\langle \mathcal A^t\left( u_0v_0^T \right),z^t \right\rangle$ and $H =  \sum_{t=1}^dw_t \left\langle \mathcal A^t\left( u_0v_0^T \right),h^t\right\rangle$. Since for all $1\le t\le d$, entries of $z^t$ are i.i.d.\ $\mathcal N\left( 0, \sigma_1^2\right) $, therefore $Z \sim \mathcal N\left( 0,\sigma_Z^2\right)$, where the variance $\sigma_Z^2$ is
\begin{equation}
\label{eq:var-z}
\sigma_Z^2 = \sum_{t=1}^d w_t^2 \norm{\mathcal A^t\left( u_0v_0^T\right)}_2^2\sigma_1^2 \le \sum_{t=1}^d w_t^2 \left( 1+\delta_1\right) \norm{u_0v_0^T}_F^2\sigma_1^2  = \sum_{t=1}^d w_t^2 \left( 1+\delta_1\right) \sigma_1^2.
\end{equation}
The first inequality uses the matrix RIP for rank-$1$ matrices. For a fixed $t$, $\mathcal A^t$ satisfies the rank-$1$ matrix RIP with constant $\delta_1$, with probability at least $1-C_2\exp(-c_2m_0)$  provided that $m_0 \ge D_2n$ by Theorem 2.3 in \cite{candes2011tight}, where $C_2,c_2$ and $D_2$ are fixed positive constants. Then by a union bound, for all $1\le t\le d$, $\mathcal A^t$ satisfies the rank-$1$ matrix RIP property with parameter $\sigma_1$, with probability at least $1-dC_2\exp(-c_2m_0)$ provided that $m_0 \ge D_2 n$.

We now simplify $H$ as
\begin{equation*}
\begin{split}
H =  \sum_{t=1}^dw_t \left\langle \mathcal A^t\left( u_0v_0^T \right),h^t\right\rangle &= \sum_{t=1}^{d-1}w_t \left\langle \mathcal A^t\left( u_0v_0^T \right),\sum_{s= t+1}^d \mathcal A^t\left[U \left( \epsilon^s\right)^T \right] \right\rangle \\
& = \sum_{s=2}^{d}  \sum_{t=1}^{s-1} \left\langle w_t\mathcal A^t\left(u_0v_0^T\right) , \mathcal A^t\left[U \left( \epsilon^s\right)^T \right] \right\rangle \\
& =  \sum_{s=2}^{d}  \sum_{t=1}^{s-1}   \left\langle w_t \mathcal A^{t*}\mathcal A^t\left(u_0v_0^T\right) , U \left( \epsilon^s\right)^T \right\rangle \\
& = \sum_{s=2}^{d}  \sum_{t=1}^{s-1}  \sum_{i=1}^{m_0}  \left\langle w_t \left[ \mathcal A^t\left(u_0v_0^T\right) \right]_i A^t_i, U \left( \epsilon^s\right)^T \right\rangle \\
& = \sum_{s=2}^{d}   \left\langle  \sum_{t=1}^{s-1}  w_t \norm{ \mathcal A^t\left(u_0v_0^T\right) }_2 U^TA^t,  \left( \epsilon^s\right)^T \right\rangle, \\
\end{split}
\end{equation*}
where $A^t\in\mathbb R^{n_1\times n_2}$ contains i.i.d.\ $\mathcal N\left(0,1/m_0\right)$ entries. The last equality uses the property that sum of independent Gaussian variables is also Gaussian, and the variance is the sum of individual variances. Since for all $2\le s\le d$, entries of $\epsilon^s$ are i.i.d.\ $\mathcal N \left( 0 , \sigma_2^2 \right)$, therefore $H\sim \mathcal N\left( 0, \sigma_H^2\right) $, where the variance $\sigma_H^2$ is
\begin{equation}
\label{eq:var-h}
\begin{split}
\sigma_H^2 = \sum_{s=2}^{d}   \norm{  \sum_{t=1}^{s-1} w_t \norm{ \mathcal A^t\left(u_0v_0^T\right) }_2 U^TA^t }_F^2\sigma_2^2
& \overset{(\xi_1)}{\le} \sum_{s=2}^{d}   \norm{  \sum_{t=1}^{s-1} w_t\sqrt{1+\delta_1} U^TA^t }_F^2\sigma_2^2  \\
& \overset{(\xi_2)}{=} \sum_{s=2}^{d}    \sum_{t=1}^{s-1} { w_t^2 \left(1+\delta_1\right)}  \norm{U^TB^s}_F^2 \sigma_2^2 \\
& = \sum_{s=2}^{d}   \sum_{t=1}^{s-1} { w_t^2 \left(1+\delta_1\right)}\frac{1}{m_0}\chi^2_s\left(r n_2\right) \sigma_2^2 \\
& \overset{(\xi_3)}{\le}  \sum_{s=2}^{d}   \sum_{t=1}^{s-1} { w_t^2 \left(1+\delta_1\right)}\frac{1}{m_0}3m_0 \sigma_2^2 \\
& = \sum_{t = 1}^{d-1}  (d-t) w^2_t \left( 1+\delta_1\right)  \sigma_2^2.
\end{split}
\end{equation}
Inequality $(\xi_1)$ holds with probability exceeding  $1-dC_2\exp(-c_2m_0)$ provided that $m_0\ge Dn$ based on the matrix RIP for rank-$1$ matrices as used while bounding $\sigma_Z^2$. Equality $(\xi_2)$ uses the property that sum of independent Gaussian variables is also Gaussian and entries of $B^s$ are i.i.d.\ $\mathcal N(0,1/m_0)$. Inequality $(\xi_3)$ holds with probability at least $1- dC_3\exp(-c_3m_0)$ by the concentration property of correlated Chi-squared variables.

Since the measurement noise $Z$ and dynamic perturbation $H$ are independent, then $ \left\langle u_0, Wv_0\right\rangle  \sim \mathcal N \left( 0, \sigma_Z^2 + \sigma_H^2\right)$. Then by a standard tail bound for Gaussian random variables we have
\begin{equation*}
\mathbb P \left( \left| \left\langle u_0, Wv_0\right\rangle \right| > \lambda \right) \le 2 \exp\left(- \frac{\lambda^2}{2\left( \sigma_H^2 + \sigma_Z^2 \right) }\right).
\end{equation*}
Therefore by an standard union bound we bound the stochastic error
\begin{equation}
\label{eq:stoch-bound}
\mathbb P\left( \norm{W}_2  \ge C_0 \sqrt{n \left( \sigma_H^2 + \sigma_Z^2 \right)} \right) \le 2 \left| \mathcal N _{1/4}\right|^2 \exp\left( - \frac{C_0^2 n}{8 } \right) \le  2  \exp\left(-cn\right),
\end{equation}
where $c  =  \frac{C_0^2}{8}  - 2\log 12$. To ensure $c >0$, we require $C_0 > 4\sqrt{\log 12}$.

Combining \eqref{eq:var-z}, \eqref{eq:var-h}, and \eqref{eq:stoch-bound}, if $m_0 \ge  {Dn}$ we have
\begin{equation*}
\norm{W}_2  \le C_0 \sqrt{n \left( (1+\delta_1)\sum_{t=1}^d w_t^2 \left(\sigma_1^2 + (d-t)\frac{5rn_2}{m_0}\sigma_2^2 \right) \right)}
\end{equation*}
with probability exceeding $1 - [dC_2\exp(-c_2m_0)+ dC_3\exp(-c_3m_0)+ 2\exp(-cn)] \ge 1- dC\exp(-cn_2) $.
\end{proof}

\section{Proof of Theorem \ref{thm:matrix-completion}}
\begin{proof}
The proof follows the same framework of the proof of Theorem 7 in \cite{klopp2014noisy}.

Before we lower bound $\sum_{t=1}^d w_t \norm{\mathcal A^t\left( \Delta^d\right)}_2^2$, we consider the following constraint set for a given $0 < r \le n$:
\begin{equation*}
\label{eq:constrain-set}
\mathcal E\left( r\right)  = \left\{ X\in\mathbb C(r): \norm{X}_\infty = 1, \norm{X}_F^2 \ge n_1 n_2 \sqrt{\frac{64 \sum_{t=1}^d w_t^2\log(n_1+n_2)}{\log(6/5)m_0}} \right\}.
\end{equation*}
Define the following random matrix
\begin{equation*}
\Sigma_R = \sum_{t=1}^d \sum_{i=1}^{m_0} w_t\gamma_i^t A^t_i,
\end{equation*}
where $\gamma_i^t$ is Rademacher variable.

The following lemma bounds the restricted strong convexity (see \cite{negahban2012restricted}) of the operator $\left\{\sqrt{w_t}\mathcal A^t\right\}_{t=1}^d $.
\begin{lemma}
\label{lemma:restricted-convexity}
Suppose all $\mathcal A^t$'s are fixed uniform sampling ensembles. For all $X\in \mathcal E\left( r\right) $
\begin{equation}
\sum_{t=1}^d w_t \norm{\mathcal A^t\left( X\right)}_2^2 \ge \frac{p}{2} \norm{X}_F^2 -\frac{ 44 r n_1 n_2}{m_0} \left( \mathbb E(\norm{\Sigma_R}) \right)^2
\end{equation}
with probability at least $1-\frac{2}{\left( n_1 + n_2 \right) }$.
\end{lemma}
\begin{proof}
See Appendix \ref{sec:restricted-convexity}.
\end{proof}

Note that $\norm{\Delta^d}_\infty \le \norm{\hat{X}^d}_\infty + \norm{X^d}_\infty \le 2\norm{X^d}_\infty $. To proceed, we consider the following two cases.

\emph{Case I.} $\frac{\Delta^d}{2\norm{X^d}_\infty} \notin \mathcal E(2r) $.

Following the definition of $ \mathcal E(2r)$ we have
\begin{equation*}
\norm{\Delta^d}_F^2 \le  c_2 \norm{X^d}^2_\infty n_1n_2 \sqrt{\frac{\sum_{t=1}^d w_t^2\log(n_1+n_2)}{m_0}},
\end{equation*}
where $C_2 = 4\sqrt{\frac{64}{\log(6/5)}}$. This yields the first part of inequality \eqref{eq:bound-completion} in Theorem \ref{thm:matrix-completion}.

\emph{Case II.} $\frac{\Delta^d}{2\norm{X^d}_\infty} \in \mathcal E(2r) $.

Since $\frac{\Delta^d}{2\norm{X^d}_\infty} \in \mathcal E(2r)$, applying Lemma \ref{lemma:restricted-convexity} yields
\begin{equation}
\label{eq:rest_con}
\sum_{t=1}^d w_t \norm{\mathcal A^t\left( \Delta^d \right)}_2^2 \ge \frac{p}{2} \norm{\Delta^d}_F^2 - \frac{ 362 r n_1 n_2}{m_0}  \left( \mathbb E(\norm{\Sigma_R}) \right)^2 \norm{X^d}_\infty^2.
\end{equation}
Combining \eqref{eq:rest_con} and \eqref{eq:basic-ineq-formal} yields
\begin{equation*}
\begin{split}
\frac{p}{2} \norm{\Delta^d}_F^2  &\le 2\sqrt{2r} \norm{ \sum_{t=1}^{d} w_t \mathcal A^{t*} \left(h^t - z^t \right)}_2 \norm {\Delta^d  }_F +  \frac{ 362 r n_1 n_2}{m_0}  \left( \mathbb E(\norm{\Sigma_R}) \right)^2 \norm{X^d}_\infty^2\\
& \le  \frac{8r}{p}\norm{ \sum_{t=1}^{d} w_t \mathcal A^{t*} \left(h^t - z^t \right)}_2 ^2  + \frac{p}{4}  \norm{\Delta^d}_F^2 + \frac{ 362 r n_1 n_2}{m_0}  \left( \mathbb E(\norm{\Sigma_R}) \right)^2 \norm{X^d}_\infty^2.
\end{split}
\end{equation*}
The above inequality can be further simplified as
\begin{equation}
\label{eq: completion-bound1}
\norm{\Delta^d}_F^2  \le \frac{32r n^2_1 n^2_2}{m_0^2}\norm{ \sum_{t=1}^{d} w_t \mathcal A^{t*} \left(h^t - z^t \right)}_2 ^2  + \frac{1448 r n^2_1 n^2_2}{m^2_0} \left( \mathbb E(\norm{\Sigma_R}) \right)^2 \norm{X^d}_\infty^2.
\end{equation}

Next we bound $\mathbb E(\norm{\Sigma_R}) $  in the following lemma.
\begin{lemma}
\label{lemma:exp_rademacher}
Suppose all $\mathcal A^t$'s are fixed uniform sampling ensembles. For $m_0 \ge D n_{\min} \log\left( n_1 + n_2 \right)\phi(w)$, where $\phi(w) = \frac{w^2_{\max}}{\sum_{t=1}^dw_t^2}$, there exists an absolute positive constant $C$ such that
\begin{equation}
\label{eq:completion-bound2}
\mathbb E(\norm{\Sigma_R})  \le C \sqrt{\frac{2e\log{(n_1 + n_2)}\sum_{t=1}^d w_t^2 m_0}{n_{\min}}}.
\end{equation}
\end{lemma}
The proof is not provided since it is almost the same as that of Lemma 6 in \cite{klopp2014noisy} with some minor modifications. Note that our results are a bit stronger compared to Lemma 6 in \cite{klopp2014noisy}, since we are dealing with bounded variables.

Now we upper bound the stochastic error $\norm{J}_2^2 \defeq \norm{ \sum_{t=1}^{d} w_t \mathcal A^{t*} \left(h^t-z^t \right)}^2_2 $. First, we rewrite $J$ as
\begin{equation*}
J = \sum_{t=1}^d w_t \mathcal A^{t*} \mathcal A^t \left[U \left( \sum_{s=t+1}^d \epsilon^s\right)^T + Z^t\right],
\end{equation*}
where each entry of the random matrix $Z^t \in \mathbb R^{n_1\times n_2}$ is i.i.d.\ Gaussian distributed with variance $\sigma_1^2$. Set $Y^t = U \left( \sum_{s=t+1}^d \epsilon^s\right)^T$ and $F^t = Y^t +Z^t$. Note that $F^t$ may be correlated for different $1\le t\le d$, though for a given $t$ the entries of $F^t$ are independent.

We now introduce an $n_1\times n_2$ random matrix $G^t$ that has exactly one non-zero entry:
\begin{equation*}
G^t = w_t n_1 n_2 F^t_{ij}E_{ij}, \quad \textrm{with probability }\frac{1}{n_1 n_2},
\end{equation*}
where $E_{ij}$ is the canonical basis of matrices with dimension $n_1\times n_2$. We also introduce the following random matrix $H^t$, which is the average of $m_0$ independent copies of $G^t$:
\begin{equation*}
H^t  = \frac{1}{m_0}\sum_{i=1}^{m_0} G^t_i \quad \textrm{where each }G^t_i\textrm{ is an independent copy of }G^t.
\end{equation*}
Then $J$ can be decomposed as sum of independent random matrices:  $J = \frac{m_0}{n_1n_2}\sum_{t=1}^d H^t $. It is immediate that
\begin{equation*}
\mathbb E G^t = \mathbb E H^t = w_t F^t, \quad \mathbb E J = \frac{m_0}{n_1 n_2} \sum_{t=1}^d w_t F^t.
\end{equation*}

Before we proceed we introduce a lemma describing the spectral norm deviation of a sum of uncentered random matrices from its mean value.
\begin{lemma}(Corollary 6.1.2 in \cite{tropp2015introduction})
\label{lemma:uncentered-bernstein}
Consider a finite sequence $\{S_k\}$ of independent random matrices with common dimension $n_1\times n_2$. Assume that each matrix has uniformly bounded deviation from its mean:
\begin{equation*}
\norm{S_k - \mathbb ES_k} \le L \quad \textrm{for each index }k.
\end{equation*}
Consider the sum
\begin{equation*}
Z = \sum_k S_k.
\end{equation*}
Let $\rho(Z)$ denotes the matrix variance statistic of the sum:
\begin{equation*}
\begin{split}
\rho(Z) &= \max \left\{ \norm{\mathbb E [(Z-\mathbb E Z)(Z-\mathbb E Z)^T]}, \norm{\mathbb E [(Z-\mathbb E Z)^T(Z-\mathbb E Z)]} \right\} \\
& = \max \left\{ \norm {\sum_k \mathbb E [(S_k-\mathbb E S_k)(S_k-\mathbb E S_k)^T]}, \norm {\sum_k \mathbb E [(S_k-\mathbb E S_k)^T(S_k-\mathbb E S_k)]} \right\}.
\end{split}
\end{equation*}
Then for all $s \ge 0$,
\begin{equation*}
\mathbb P \left( \norm{Z - \mathbb E Z} \ge s\right) \le (n_1 +n_2) \exp\left( \frac{-s^2/2}{\rho(Z)+Ls/3} \right).
\end{equation*}
\end{lemma}

We are going to apply the above uncentered Bernstein inequality to the sum of $dm_0$ independent random matrices $\sum_{t=1}^d H^t= \frac{1}{m_0}\sum_{t=1}^d \sum_{k=1}^{m_0}G_k^t$.  Before doing so, we note that for given $t$ and $k$,
\begin{equation*}
\norm{G_k^t - \mathbb E G_k^t}  \le \norm{G_k^t} + \norm{ \mathbb E G_k^t} \le \norm{G_k^t}  + \mathbb E \norm{  G_k^t} \le 2 \norm{G_k^t}.
\end{equation*}
The first inequality uses the triangle inequality; the second is Jensen's inequality.

To control $\rho(\sum_{t=1}^d H^t)$, first note that
\begin{equation*}
\begin{split}
\mathbf 0 \preceq  {\sum_{t}\sum_k \mathbb E \left[ G_k^t- \mathbb E G_k^t)(G_k^t- \mathbb E G_k^t)^T \right] }  &=  {\sum_{t}\sum_k \mathbb E \left[ (G_k^t (G_k^t)^T \right] -  (\mathbb E G_k^t) (\mathbb E G_k^t)^T } \\
&\preceq  {\sum_{t}\sum_k \mathbb E \left[ G_k^t (G_k^t)^T \right] } \\
&= m_0  {\sum_{t} \mathbb E \left[ G^t (G^t)^T \right] }.
\end{split}
\end{equation*}

The third relation holds because $ (\mathbb E G_k^t) (\mathbb E G_k^t)^T$ is positive semidefinite; the last relation uses the fact that for a fixed $t$, $G_k^t$ are random matrices following identical distributions independently for all $1\le k \le m_0$. Now we can control $\rho(\sum_{t=1}^d H^t)$ in the following
\begin{equation*}
\begin{split}
\rho\left(\sum_{t=1}^d H^t \right) \le \frac{1}{m_0} \max \left\{ \norm {\sum_{t} \mathbb E \left[ (G^t (G^t)^T \right] }, \norm {\sum_{t} \mathbb E \left[ (G^t)^T G^t \right] } \right\}.
\end{split}
\end{equation*}
Set $\rho_0  \defeq \max\left\{ \norm{\sum_{t=1}^d\mathbb E (G^t (G^t)^T)}, \norm{\sum_{t=1}^d \mathbb E ( (G^t)^T G^t) } \right\}$.  Then the remaining work is to uniformly upper bound $\norm{G_k^t}$ for all $1\le t \le d$ and $1\le k \le m_0$ and upper bound $\rho_0$.

First we turn to the uniform bound on the spectral norm of the random matrix $G^t_k$ for all $1\le t \le d $ and $1\le k \le m_0$. We have for all $1\le t \le d $ and $1\le k \le m_0$
\begin{equation*}
\norm{G_k^t} \le \max_{i,j,t}w_t\norm{n_1 n_2 F^t_{ij} E_{ij}} =  n_1 n_2\max_{i,j,t}w_t |F^t_{ij}|.
\end{equation*}
Since $\mu(U)\le \mu_0$, the variance of each entry of the random matrix $F^t$ can be bounded as $\textrm{Var}(F^t_{ij})\le \frac{\mu_0^2r}{n_1}\sigma_2^2 (d-t) + \sigma_1^2$. Let $\sigma^2_{\max} = \max_{t} w^2_t \left( \frac{\mu_0^2r}{n_1}\sigma_2^2 (d-t) + \sigma_1^2\right)$. Then by the tail probability of Gaussian random variables and the standard union bound (over $i,j$),  for all $1\le t \le d $ and $1\le k \le m_0$ we have
\begin{equation*}
\mathbb P \left( \norm{G_k^t} \le  n_1n_2\sqrt{2\log( d(n_1+n_2)n_1n_2) \sigma^2_{\max} } \eqdef L \right) \ge 1-2/(n_1+n_2).
\end{equation*}

Second we turn to the computation of $\mathbb E (G^t (G^t)^T)$. We have
\begin{equation*}
\mathbb E (G^t (G^t)^T) = w_t^2 n_1^2 n_2^2 \sum_{i = 1}^{n_1} \sum_{j=1}^{n_2}(F^t_{ij})^2 E_{ij} E_{ij}^T \frac{1}{n_1n_2} = w_t^2 n_1 n_2 \sum_{i = 1}^{n_1} \sum_{j=1}^{n_2}(F^t_{ij})^2 E_{ii}.
\end{equation*}
Similarly $\mathbb E ( (G^t)^T G^t) = w_t^2 n_1 n_2 \sum_{i = 1}^{n_1} \sum_{j=1}^{n_2}(F^t_{ij})^2 E_{jj}$. Then
\begin{equation*}
\begin{split}
\rho = n_1 n_2 \max \left\{ \max_i \sum_{t=1}^d \sum_{j=1}^{n_2} w_t^2(F^t_{ij})^2, \max_j\sum_{t= 1}^d\sum_{i=1}^{n_1} w_t^2(F^t_{ij})^2\right\}. \\
\end{split}
\end{equation*}
Let $a_i = \sum_{t=1}^d \sum_{j=1}^{n_2} w_t^2(F^t_{ij})^2$ and $b_j = \sum_{t= 1}^d\sum_{i=1}^{n_1} w_t^2(F^t_{ij})^2$. We first bound $\max_i a_i$. Note that  $a_i = \sum_{t=1}^d  w_t^2 \sum_{j=1}^{n_2}(Y^t_{ij} + Z^t_{ij})^2 \le2 \sum_{t=1}^d  w_t^2 \sum_{j=1}^{n_2}[(Y^t_{ij})^2 + (Z^t_{ij})^2]$. Note that for $1\le i \le n_1$ and $1\le t \le d$, $\sum_{j=1}^{n_2} (Z_{ij}^t)^2 \sim \sigma_1^2\chi^2(n_2)$ and are independent. So by the tail bound of Chi-squared variable and the standard union bound (over $i$ and $t$) we have
\begin{equation}
\label{eq:second-moment-b1}
\mathbb P \left( \max_{i} \sum_{t=1}^dw_t^2 \sum_{j=1}^{n_2} (Z_{ij}^t)^2 \le 5n_2\sum_{t=1}^d w_t^2\sigma_1^2 \right) \ge 1- dn_1 \exp(-n_2).
\end{equation}
Similarly we have
\begin{equation}
\label{eq:second-moment-b2}
\mathbb P \left( \max_{j} \sum_{t=1}^dw_t^2 \sum_{i=1}^{n_2} (Z_{ij}^t)^2 \le 5n_1\sum_{t=1}^d w_t^2\sigma_1^2 \right) \ge 1- dn_2 \exp(-n_1).
\end{equation}
For $\sum_{j=1}^{n_2} (Y_{ij}^t)^2 $, note that $Y_{ij}^t$ is Gaussian distributed and the variance is not greater than $\frac{\mu^2_0r}{n_1}(d-t)\sigma_2^2 $ for all $i,j,t$, since $\mu(U)\le \mu_0$. For a fixed $i$, for all $1\le j \le n_2$, $Y_{ij}^t$ are independent Gaussian random variables. So given $i$ and $t$, applying the tail bound of Chi-squared random variables yields
\begin{equation*}
\mathbb P \left( \sum_{j=1}^{n_2} (Y_{ij}^t)^2 \le 5n_2 (d-t)\frac{\mu^2_0r}{n_1}\sigma_2^2 \right) \ge 1 - \exp(-n_2).
\end{equation*}
By the standard union bound (over $i$ and $t$) we have
\begin{equation}
\label{eq:second-moment-b3}
\mathbb P \left( \max_i \sum_{t=1}^d  w_t^2 \sum_{j=1}^{n_2} (Y_{ij}^t)^2 \le 5n_2\frac{\mu^2_0r}{n_1}\sum_{t=1}^d(d-t)w_t^2 \sigma_2^2 \right) \ge 1- dn_1\exp(-n_2).
\end{equation}
Now we turn to $\sum_{i=1}^{n_1} (Y_{ij}^t)^2$, which follows a Chi-squared distribution $(d-t)\sigma_2^2\chi^2(r)$, since
\begin{equation*}
\begin{split}
\sum_{i=1}^{n_1} (Y_{ij}^t)^2 = (Y^t_{:j})^TY^t_{:j} &= \bar{\epsilon}^t_{j:} U^T U \left(\bar{\epsilon}^t_{j:}\right)^T = \bar{\epsilon}^t_{j:}  \left(\bar{\epsilon}^t_{j:}\right)^T
\end{split}
\end{equation*}
where $\bar{\epsilon}^t = \sum_{s=t+1}^d\epsilon^s$. The last equality uses the fact that $U$ is orthonormal. Then by the tail bound of Chi-squared random variables and the standard union bound (over $j$ and $t$) we have
\begin{equation}
\label{eq:second-moment-b4}
\mathbb P \left( \max_j \sum_{t=1}^d  w_t^2 \sum_{i=1}^{n_1} (Y_{ij}^t)^2 \le 5n_1\sum_{t=1}^d(d-t)w_t^2 \sigma_2^2 \right) \ge 1- dn_2\exp(-n_1).
\end{equation}

Combining \eqref{eq:second-moment-b1} and \eqref{eq:second-moment-b3} yields
\begin{equation}
\label{eq:second-moment-a_i}
\mathbb P \left( \max_{i}a_i \le 10 n_2 \sum_{t=1}^d w_t^2 \left( \sigma_1^2 + \frac{\mu_0^2r}{n_1}(d-t)\sigma_2^2\right) \right) \ge 1- 2dn_1\exp(-n_2).
\end{equation}
Similarly combining \eqref{eq:second-moment-b2} and \eqref{eq:second-moment-b4} yields
\begin{equation}
\label{eq:second-moment-b_j}
\mathbb P \left( \max_{j}b_j \le 10 n_1 \sum_{t=1}^d w_t^2 \left( \sigma_1^2 + (d-t)\sigma_2^2\right) \right) \ge 1- 2dn_2\exp(-n_1).
\end{equation}

Note that $1\le \mu_0 \le \sqrt{n_1}/\sqrt{r}$, so $\frac{\mu_0^2r}{n_1}\le 1$.  Now we are ready to bound $\rho_0$ by combining \eqref{eq:second-moment-a_i} and \eqref{eq:second-moment-b_j}:
\begin{equation}
\mathbb P \left( \rho_0 \le 10 n_{\max}n_1n_2\left( \sum_{t=1}^d w_t^2 \sigma_1^2 + \sum_{t=1}^d w_t^2 (d-t)\sigma_2^2 \right) \eqdef \nu \right) \ge 1- 4 dn_{\max}\exp(-n_{\min}).
\end{equation}

Now by Lemma \ref{lemma:uncentered-bernstein}, we have
\begin{equation*}
\mathbb P \left( \norm{\sum_{t=1}^d H^t - \sum_{t=1}^d w_tF^t} \ge s \right)  \le (n_1 + n_2) \exp\left( \frac{-m_0 s^2/2}{\nu + 2 Ls/3} \right).
\end{equation*}

If we let $s = \sqrt{\frac{8\log(n_1+n_2)\nu}{m_0}}$ and substitute this into the above matrix Bernstein inequality we obtain
\begin{equation*}
\mathbb P \left(  \norm{\sum_{t=1}^d H^t - \sum_{t=1}^d w_tF^t}  \ge \sqrt{\frac{8\log(n_1+n_2)\nu}{m_0}} \right) \le 1/(n_1+ n_2).
\end{equation*}
A hidden condition when the above inequality holds is that $\nu$ dominates the denominator of the exponential term.  The remaining work is to have sufficiently large $m_0$ to guarantee that $\nu$ dominates the denominator of the exponential, which follows
\begin{equation*}
\nu  \ge 2/3  L \sqrt{\frac{8\log(n_1+n_2)\nu}{m_0}} .
\end{equation*}
The above inequality immediately implies that
\begin{equation*}
m_0 \ge \frac{32}{45} n_{\min}\log( d(n_1+n_2)n_1n_2) \log(n_1+n_2) \frac{ \max_{t} w_t^2 \left( (d-t)\frac{\mu_0^2 r}{n_1}\sigma_2^2  + \sigma_1^2\right)} {\sum_{t=1}^d w_t^2 \left( (d-t)\sigma_2^2  + \sigma_1^2 \right) }.
\end{equation*}
Note that $n_1 + n_2 > n_i, i=1,2$, and $n_1+n_2 > d$, then the above sample complexity can be simplified as
\begin{equation}
\label{completion-stochastic-sample-complexity}
m_0 \ge \frac{128}{45}n_{\min}\log^2(n_1+n_2) \frac{ \max_{t} w_t^2 \left( (d-t)\frac{\mu_0^2 r}{n_1}\sigma_2^2  + \sigma_1^2\right)} {\sum_{t=1}^d w_t^2 \left( (d-t)\sigma_2^2  + \sigma_1^2 \right) }.
\end{equation}

The remaining work is to bound $\norm{\sum_{t=1}^d w_tF^t}$. First we note that each entry of $F^t$ is Gaussian and the variance is not greater than $\sigma_1^2 + (d-t)\sigma_2^2$. Then, according to results on bounds for the spectral norm of i.i.d.\ Gaussian ensemble, we have
\begin{equation}
\label{gaussian-matrix=spectral-norm}
\mathbb P \left( \norm{\sum_{t=1}^d w_tF^t} \le 2\sqrt{\sum_{t=1}^d w_t^2\left( \sigma_1^2 + (d-t)\sigma_2^2\right) }\sqrt{n_{\max}} \right) \ge 1 - C_1\exp(-c_2n_{\max}),
\end{equation}
where $C_1, c_2$ are absolute positive constants. Note that $C_1\exp(-c_2n_{\max}) \ll d n_{\max}\exp(-n_{\min})$.

Now we are ready to bound $\norm{J}_2^2$. With probability at least $1 - \frac{3}{n_1+n_2} - 5d n_{\max}\exp(-n_{\min})$ we have
\begin{equation}
\label{eq:upper_stoch_error_comp}
\begin{split}
 \norm{J}^2_2 & \le p^2 \left( \norm{\sum_{t=1}^d w_tF^t} + \sqrt{\frac{8\log(n_1+n_2)\nu}{m_0}} \right)^2  \\
 & \le 320 p^2 \max\{n_1n_2\log(n_1+n_2)/m_0, 1\}n_{\max}\sum_{t=1}^d w_t^2 ((d-t)\sigma_2^2 + \sigma_1^2) \\
 & = 320 p^2 \sum_{t=1}^d w_t^2 ((d-t)w_2^2 + \sigma_1^2) n_1n_2\log(n_1+n_2)n_{\max}/m_0 \\
 &  = \frac{ 320 m_0  \log(n_1 + n_2) \sum_{t=1}^d w_t^2 ((d-t)\sigma_2^2 + \sigma_1^2)} {n_{\min}}.
\end{split}
\end{equation}
The first equality uses the fact that $m_0 < n_1 n_2 \log(n_1+n_2)$.

Combining \eqref{eq: completion-bound1},\eqref{eq:completion-bound2} and \eqref{eq:upper_stoch_error_comp}  yields the second part of inequality \eqref{eq:bound-completion} in Theorem \ref{thm:matrix-completion}.

\end{proof}

\section{Proof of Lemma \ref{lemma:restricted-convexity}}
\label{sec:restricted-convexity}
\begin{proof}
The proof is almost the same as the proof of Lemma 12 in \cite{klopp2014noisy} with some minor modifications.

Set $\mathcal F = \frac{44 r n_1 n_2}{m_0} \left( \mathbb E(\norm{\Sigma_R}) \right)^2$. We will show that the probability of the following bad event is small:
\begin{equation*}
\mathcal B = \left\{ \exists X \in \mathcal E (r) \textrm{ such that } \left| \sum_{t=1}^d w_t \norm{\mathcal A^t\left( X\right)}_2^2  - p \norm{X}_F^2 \right| > \frac{p}{2} \norm{X}_F^2  + \mathcal F \right\}.
\end{equation*}
Note that $\mathcal B$ contains the complement of the event in Lemma \ref{lemma:restricted-convexity}.

We use a peeling argument to bound the probability of $\mathcal B$. Let $\nu = \sqrt{\frac{64 \sum_{t=1}^d w_t^2\log(n_1+n_2)}{\log(6/5)m_0}}$ and $\alpha = 6/5$. For $l \in \mathcal N$ let
\begin{equation*}
S_l = \left\{ X \in \mathcal E (r): \nu \alpha ^{l-1} \le \frac{1}{n_1n_2} \norm{X}_F^2 \le \nu \alpha ^{l} \right\}.
\end{equation*}
Then if event $\mathcal B$ holds for some $X\in \mathcal E(r)$, it must be that $X$ belongs to some $S_l$ and
\begin{equation}
\label{eq:peeling}
\left| \sum_{t=1}^d w_t \norm{\mathcal A^t\left( X\right)}_2^2  - p \norm{X}_F^2 \right|  > \frac{p}{2} \norm{X}_F^2 + \mathcal F > \frac{5}{12}\alpha^l \nu m_0 + \mathcal F.
\end{equation}
For $T > \nu$ consider the set
\begin{equation*}
\mathcal E(r,T) = \left\{ X \in \mathcal E(r): \norm{X}_F^2 \le n_1n_2 T\right\}
\end{equation*}
and the event
\begin{equation}
\mathcal B_l = \left\{ \exists X \in \mathcal E (r,\alpha^l\nu) \textrm{ such that } \left| \sum_{t=1}^d w_t \norm{\mathcal A^t\left( X\right)}_2^2  - p \norm{X}_F^2 \right| > \frac{5}{12}\alpha^l \nu m_0  + \mathcal F \right\}.
\end{equation}
Note that $X\in S_l$ implies that $X\in\mathcal E(r, \alpha^l \nu)$. Then \eqref{eq:peeling} implies that $\mathcal B_l$ holds and  $\mathcal B \subset \cup \mathcal B_l$. Thus, it is sufficient to bound the probability of the simpler event $\mathcal B_l$ and then apply the union bound. Such a bound is given by the following lemma. Its proof is given in Appendix \ref{sec:simpler-event}. Let
\begin{equation*}
H_T  = \sup_{X\in \mathcal E(r,T)}\left| \sum_{t=1}^d w_t \norm{\mathcal A^t\left( X\right)}_2^2  - p \norm{X}_F^2 \right|.
\end{equation*}
\begin{lemma} Suppose all $\mathcal A^t$'s are fixed uniform sampling ensembles. Then
\label{lemma:simpler-event}
\begin{equation*}
\mathbb P \left( H_T  >  \frac{5}{12}\alpha^l \nu m_0 + \mathcal F \right) \le \exp\left(\frac{-c_5m_0T^2 }{\sum_{t=1}^dw_t^2}\right),
\end{equation*}
where $c_5 = 1/128$.
\end{lemma}
The above lemma implies that $\mathbb P(\mathcal B_l) \le \exp(-c_5 m_0\alpha^{2l}\nu^2)$. By a union bound, we have
\begin{equation*}
\mathbb P (\mathcal B) \le \sum_{l=1}^\infty \mathbb P(\mathcal B_l) \le \sum_{l=1}^\infty \exp\left(\frac{-c_5 m_0\alpha^{2l}\nu^2}{\sum_{t=1}^d w_t^2}\right) \le \sum_{l=1}^\infty \exp\left(\frac{-(2c_5 m_0\alpha\nu^2)l}{\sum_{t=1}^d w_t^2}\right),
\end{equation*}
where the last inequality uses the bound $e^x\ge x$. Then, substituting $v =\sqrt{\frac{64\sum_{t=1}^dw_t^2\log (n_1+n_2)}{\log(6/5)m_0}}$ into the above summation we obtain
\begin{equation*}
\mathbb P(\mathcal B) \le 2/(n_1+n_2).
\end{equation*}
This completes the proof.
\end{proof}

\section{Proof of Lemma \ref{lemma:simpler-event}}
\label{sec:simpler-event}
\begin{proof}
The proof is almost the same as the proof of Lemma 14 in \cite{klopp2014noisy} with some minor modifications.

By Massart's concentration inequality (see, e.g., \cite{buhlmann2011statistics}, Theorem 14.2), we have
\begin{equation}
\label{eq:massart-ineq}
\mathbb P \left( H_T \ge \mathbb E(H_T) + \frac{1}{9}\frac{5}{12}T \right) \le \exp\left(\frac{-c_5m_0T^2 }{\sum_{t=1}^dw_t^2}\right) ,
\end{equation}
where $c_5 = 1/128$. Next we bound the expectation $\mathbb E(H_T)$. Using a symmetrization argument we obtain
\begin{equation*}
\mathbb E(H_T) \le 2\mathbb E\left(  \sup_{X\in \mathcal E(r,T)} \left| \sum_{t=1}^d w_t \gamma_i^t\sum_{i=1}^{m_0} \left\langle A^t_i , X \right\rangle^2 \right|  \right),
\end{equation*}
where $\gamma_i^t$ is a Rademacher variable (independent on both $i$ and $t$). The assumption $\norm{X}_\infty =1$ implies that $\left| \left\langle A_i^t, X\right\rangle\right| \le 1$. Then the contraction inequality yields
\begin{equation*}
\mathbb E(H_T) \le 8 \mathbb  E\left( \sup_{X\in \mathcal E(r,T)}\left| \sum_{t=1}^d w_t \gamma_i^t\sum_{i=1}^{m_0} \left\langle A^t_i , X \right\rangle \right|   \right) = 8\mathbb E\left( \sup_{X\in \mathcal E(r,T)} \left|  \left\langle \Sigma_R, X\right \rangle \right| \right),
\end{equation*}
where $\Sigma_R = \sum_{t=1}^d \sum_{i=1}^{m_0} w_t\gamma_i^t A^t_i$. Since $X\in \mathcal E(r,T)$, we have
\begin{equation*}
\norm{X}_* \le \sqrt{r} \norm{X}_F \le \sqrt{rn_1n_2T}.
\end{equation*}
Then by the trace duality inequality, we obtain
\begin{equation*}
\mathbb E(H_T) \le 8 \sqrt{rn_1n_2T}\mathbb E\norm{\Sigma_R}_2.
\end{equation*}
Finally using
\begin{equation*}
\frac{1}{9}\frac{5}{12}T + 8 \sqrt{rn_1n_2T}\mathbb E\norm{\Sigma_R}_2 \le \frac{1}{9}\frac{5}{12}T  + \frac{8}{9}\frac{5}{12}T + 44rn_1n_2 \left(\mathbb E\norm{\Sigma_R}_2 \right)^2
\end{equation*}
combined with \eqref{eq:massart-ineq} we complete the proof.
\end{proof}

\end{document}